\documentclass[11pt, a4paper]{berkeley-paul}

\usepackage{amsmath,amsfonts,bm}

\def\eqref#1{equation~\ref{#1}}

\def\1{\bm{1}}

\DeclareMathAlphabet{\mathsfit}{\encodingdefault}{\sfdefault}{m}{sl}
\SetMathAlphabet{\mathsfit}{bold}{\encodingdefault}{\sfdefault}{bx}{n}

\DeclareMathOperator*{\argmax}{arg\,max}

\usepackage{hyperref}
\usepackage{url}
\usepackage{booktabs}       %
\usepackage{amsfonts}       %
\usepackage{nicefrac}       %
\usepackage{microtype}      %
\usepackage[table]{xcolor}         %
\usepackage{colortbl}
\usepackage{natbib}
\usepackage[capitalise]{cleveref}
\usepackage{graphicx}
\usepackage{algorithm}
\usepackage{caption}
\usepackage{amsmath,amssymb}
\usepackage{algpseudocode}
\usepackage{wrapfig}
\usepackage{tabularx}
\usepackage{makecell}
\usepackage{amsthm}

\newtheorem{theorem}{Theorem}
\newtheorem{remark}{Remark}
\usepackage{multirow}
\usepackage{adjustbox}

\title{Reinforcement Learning via Advantage Guidance}

\author{
Zhiyuan Zhou, Sergey Levine \\ UC Berkeley
}

\newcommand{\methodname}[0]{Advantage Guidance}
\newcommand{\methodshort}{AG}

\newcommand{\ie}{\emph{i.e.},\ }

\newcommand{\naive}{na\"ive~}

\definecolor{themecolor}{HTML}{5C788F}
\definecolor{ourblue}{HTML}{73A4C6}  %
\definecolor{darkerblue}{HTML}{3B7EA1}
\hypersetup{colorlinks=true, allcolors=ourblue}

\begin{document}

\runningtitle{Reinforcement Learning via Advantage Guidance}

\teaserfigure{
    \begin{center}
        \vspace{1em}
        \includegraphics[width=\textwidth]{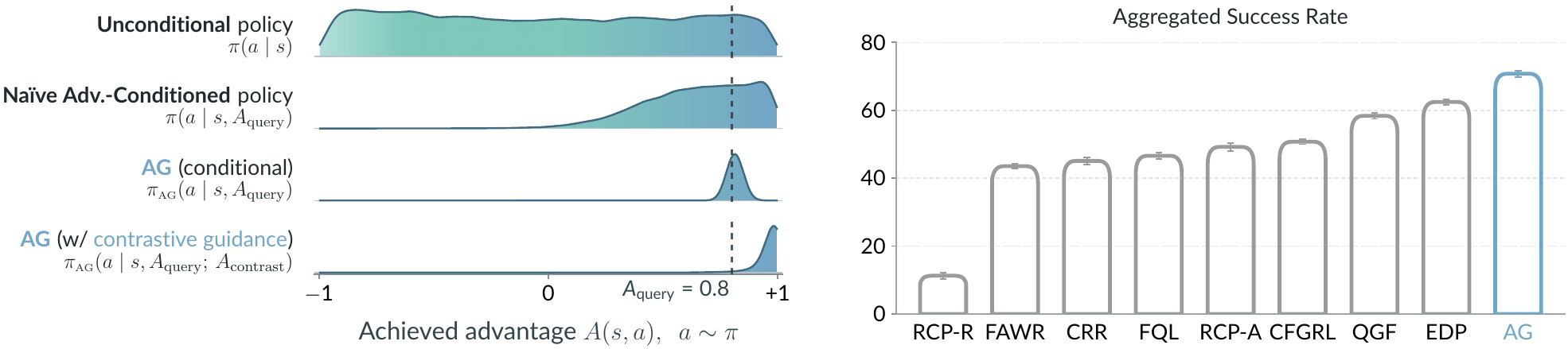}
        \captionof{figure}{We introduce a new RL method, \textcolor{darkerblue}{\textbf{\methodname{} (\methodshort{})}}, that uses a simple supervised learning objective with advantage conditioning while outperforming prior RL methods through effective conditional training and contrastive diffusion guidance during sampling. 
        \textit{Left:} Na\"ive advantage-conditioned policies often fail to match the requested advantage, producing actions whose advantages vary widely and frequently fall below the requested value. We call this the \textit{advantage alignment problem}. Our approach (AG) significantly improves advantage alignment, and our contrastive guidance formulation further enables advantage maximization rather than matching a prespecified value, as desired for policy extraction.
       \textit{Right:} Our method significantly outperforms baselines on $26$ long-horizon sparse-reward tasks across datasets with varying degrees of optimality.}
    \label{fig:teaser}
    \end{center}
}

\begin{abstract}
Reinforcement learning (RL) holds the promise of learning a policy that can outperform what is shown in the data. However, in practice, such systems suffer from instability and hyperparameter sensitivity. In contrast, supervised learning provides a stable and scalable training objective, but is fundamentally limited by the quality of the demonstrated behavior. In this paper, we ask whether we can perform policy optimization with something as simple and stable as supervised learning but as performant as RL. Unfortunately, na\"ively training the policy with supervised learning and conditioning on future rewards or advantages typically underperform dedicated RL methods. We introduce a new policy extraction method, \textcolor{darkerblue}{\textbf{\methodname{}}}, that retains the simple supervised learning structure with advantage conditioning while outperforming prior RL methods through effective conditional training and contrastive diffusion guidance during sampling. We show that \methodname{} can significantly improve a model's ability to respond to conditional advantage signals, and the contrastive guidance during sampling can use the learned advantage-conditioned policy to target the KL-regularized optimal policy. Empirically, we find that \methodname{} significantly outperforms prior state-of-the-art RL methods on 26 long-horizon, sparse reward tasks across datasets with varying degrees of optimality.
Code: \href{https://github.com/zhouzypaul/advantage-guidance}{ github.com/zhouzypaul/advantage-guidance}
\end{abstract}

\maketitle

\section{Introduction}
\label{sec:introduction}

Reinforcement learning (RL) holds the promise of improving policy performance beyond what is shown in the data, and can in theory outperform imitation learning under suitable conditions~\citep{sutton2018reinforcement, kumar2022should}.
However, RL algorithms suffer from lots of instability and hyperparameter sensitivity issues in practice~\citep{tsitsiklis1996analysis, fujimoto2018addressing, henderson2018deep}, especially over off-policy or offline data.
Imitation learning, on the other hand, has been shown to be stable and scalable for policy training with modern generative modeling techniques like diffusion and flow matching.
This raises the question: \textit{can we perform policy training with something as simple and stable as imitation learning, but as performant as RL?}

To achieve the stability of imitation learning, one might be tempted to directly apply the supervised-learning loss function, training the policy to maximize the likelihood of actions seen in the data. 
Previous works have explored accomplishing this with conditional, filtered, or weighted supervised learning~\citep{emmons2021rvs}.
One set of methods, which we will refer to as \textit{advantage-conditioned methods}~\citep{kumar2019reward, amin2025pi, gao2024act, qing2024a2po}, resorts to training a policy conditioned on advantage values, which estimate how ``good'' actions are.
During inference, the policy is queried with a high advantage to generate ``good'' actions.

Unfortunately, such \naive advantage-conditioned policies typically underperform dedicated RL methods (see \Cref{fig:teaser}).
This is in part because advantage-conditioning does not generally result in a policy that always favors actions with higher advantages as RL policies do: the resulting policy is restricted to the queried advantage level, and this restricts the degree of policy improvement.
Moreover, we find it difficult to train a policy to respond to advantage conditioning in the first place: the policy often cannot match the requested advantage (see \Cref{fig:teaser}), and can ignore the conditioning altogether and revert to unconditional imitation.

In this work, our goal is to train a policy with the simple and stable supervised learning objective conditional on continuous advantage estimates, while enabling the policy to generate actions that maximize advantage.
To accomplish this, we first derive a diffusion guidance formulation we call \textit{contrastive guidance} for sampling at test time, which can target actions that maximize the advantage under a KL behavior constraint.
Then, to effectively train a flow-matching conditional policy that responds to the conditional values, we re-parameterize the conditional policy as a residual flow on top of a reference unconditional flow field.
We also propose an \textit{action imagination} scheme: we use ``imagined actions'' sampled from the unconditional policy, and their advantage estimates, to train the advantage-conditioned policy under a broader action distribution for each state.

The resulting algorithm, \textcolor{darkerblue}{\textbf{\methodname{}}}, significantly improves the policy's ability to respond to advantage inputs, and can target the KL-regularized optimal policy during sampling. Across $26$ long-horizon, sparse reward tasks with both play-style and demonstration datasets in the offline RL setting, \methodname{} outperforms prior strong RL algorithms.

\section{Related Work}
\label{sec:related-work}

\paragraph{Offline RL and policy extraction.}
Offline RL aims to learn policies that outperform the behavior represented in a fixed dataset.
Prior works focus on learning value functions that mitigate erroneous extrapolation~\citep{an2021uncertainty, garg2023extreme, kostrikov2021offline, kumar2020conservative, xu2023offline}, as well as policy extraction from that value function~\citep{chen2022offline, fujimoto2021minimalist, hansen2023idql, nair2020awac, peng2019advantage, tarasov2024revisiting, wu2019behavior}.
However, many prior policy extraction methods can be difficult to apply in practice due to instability and hyperparameter sensitivity.
Another line of work, which closely relates to ours, circumvents this instability by converting the policy training objective to an imitation learning loss via advantage, reward, or return conditioning~\citep{kumar2019reward, chen2021decision, emmons2021rvs, srivastava2019training}, data filtering~\citep{chen2021decision, wang2020critic, chen2020bail}, and loss weighting~\citep{peng2019advantage, peters2007reinforcement, neumann2008fitted}.
We specifically study advantage conditioning, which avoids instability of loss-weighting, uses the full dataset for generalization, and has lower variance than return conditioning.
Existing advantage conditioning methods condition on continuous values~\citep{kumar2019reward, gao2024act, qing2024a2po} or binarized labels~\citep{amin2025pi, team2026n_0}, and during inference query with a fixed high advantage value, which may fall outside the training distribution, or learn an extra model to choose a high in-support query~\citep{kumar2019reward, gao2024act}.
In contrast, our method uses a form of contrastive diffusion guidance to maximize advantage without risking out-of-support advantage queries or needing additional models.
In addition, we find that previous methods struggle to respond to advantage inputs, and we introduce residual parameterization and action imagination training schemes to improve the policy's responsiveness to advantages.

\paragraph{Flow and diffusion policies for RL.}
Diffusion and flow matching have been widely studied for parameterizing the policy in RL~\citep{ren2024diffusion, ding2024diffusion, xue2025advantage, chen2024diffusion, koirala2025flow, chen2025one, zhou2026test}.
Using these iterative denoising methods is challenging because leveraging the action gradient requires backpropagating through a multi-step denoising process~\citep{wang2022diffusion, he2023diffcps, ding2023consistency, zhang2024entropy, wang2024diffusion, lv2025flow}, which can be expensive and unstable in practice~\citep{zhou2026test, park2025flow}.
Previous works avoid this unstable backpropagation process by  treating each denoising step as an action in the environment~\citep{black2023training, ren2024diffusion}, which blows up the problem horizon, or approximating the multi-step policy with single-step models~\citep{park2025flow, chen2025one, kang2023efficient}, which suffers from approximation errors and limited expressivity.
Another line of work avoids backpropagation entirely, instead directly guiding the iterative sampling process using critic gradient~\citep{zhou2026test} or classifier-free guidance (CFG)~\citep{frans2025diffusion, foffano2026advantage}.
Like the latter, our approach avoids critic gradients, which are difficult to estimate for high-dimensional or partially noised actions.
Most similar to our work is CFGRL~\citep{frans2025diffusion}, which applies CFG between an unconditional policy trained on the entire dataset and another ``filtered policy'' trained on the positive-advantage subset. Our approach instead trains on the entire dataset conditioned on continuous advantage values, and uses a form of contrastive guidance during sampling.
Empirically, we find that CFGRL usually underperforms in practice because its binary filtering criterion discards information on action quality.

\section{Preliminaries and Background}
\label{sec:preliminaries}

\paragraph{KL-regularized RL.}
We characterize the environment with state space $\mathcal{S}$, action space $\mathcal{A}$, a transition function $p(s' | s, a) : \mathcal{S} \times \mathcal{A}  \to \Delta(\mathcal{S})$, and a reward function $r(s): \mathcal{S} \rightarrow \mathbb{R}$.
A policy $\pi(a\mid s): \mathcal{S} \rightarrow \Delta(\mathcal{A})$ is a mapping from state to an action distribution, and value functions $V^\pi(s)$ and $Q^\pi(s, a)$ characterize the expected future return under a policy $\pi$. The advantage $A^\pi(s,a) = Q^\pi(s,a) - V^\pi(s)$ measures how much better action $a$ is than following $\pi$ at state $s$.
The KL-regularized RL objective is
$
    J(\pi) = \mathbb{E}_{\tau}\left[\sum_{t} \gamma^t r(s_t, a_t)\right] - 
    \beta \, \mathbb{E}_{s \sim p^{\pi}(s)}\big[D_{\mathrm{KL}}\!(\pi(\cdot \mid s) \,\|\, 
    \hat{\pi}(\cdot \mid s))\big]
$,
where $\beta$ controls regularization toward the reference policy $\hat{\pi}$.
The KL regularizer prevents the policy from exploiting $Q$-values, especially in offline RL.
It is well-known~\citep{dayan1992feudal,abdolmaleki2018maximum,peng2019advantage} that this objective under fixed-state has a closed form solution:
\begin{equation}
    \label{eqn:tilted-distribution}
    \pi_\mathrm{KL}^*(a\mid s) \propto \hat{\pi}(a\mid s) \cdot \exp (A(s, a) / \beta),
\end{equation}
which ``tilts'' the reference distribution to exponentially favor high advantage actions. We 
set this as the policy-extraction target distribution we want to achieve.

\paragraph{Advantage conditioned policies.}
The main idea for advantage-conditioned methods is straightforward: first, compute the advantage estimate $A(s, a)$ for each $(s, a)$ transition in the data via a learned value function or Monte Carlo sampling; then, train a policy $\pi(a \mid s, A(s, a))$ to predict actions given their advantage values; finally, query the policy with some specified target value $A_\mathrm{query}$.
Intuitively, this associates action quality with advantage values, allowing high-advantage queries to elicit good actions. This results in the policy
$\pi(a\mid s,A_{\mathrm{query}}) \propto
\hat{\pi}(a\mid s) \cdot \mathbf 1\!\left[|A^{\hat{\pi}}(s,a) - A_{\mathrm{qeury}}| < \epsilon \right] \; \text{where} \;\;0 < \epsilon \ll 1.$
This can be undesirable because $\pi(a | s, A_\mathrm{qeury})$ only incentivizes actions with advantage $A_\mathrm{qeury}$, but actually discourages actions with $A > A_\mathrm{qeury}$, unlike \Cref{eqn:tilted-distribution}. An excessively high $A_\mathrm{query}$ can place it outside the training distribution and result in bad actions.

\paragraph{Classifier free guidance.}
Diffusion and flow matching are generative models that transform noise into samples $x$ through an iterative denoising process $p_t(x_t)$. We use flow-matching in this paper for simplicity.
The denoising direction is closely related to the score function $\nabla_x \log p_t(x_t)$, which diffusion models learn directly and flow-matching models recover under appropriate transformations~\citep{song2020score, lipman2022flow}. 
Importantly, modifying the score during sampling allows us to guide the denoising process toward a desired distribution, such as one that favors higher-advantage actions.
For a generative model $p(x)$ with conditional labels $y$, classifier free guidance (CFG)~\citep{ho2022classifier} modifies the sampling at timestep $t$ to instead use score
$
\nabla_x\log p_t(x_t) +w\big[\nabla_x\log p_t(x_t\mid y)-\nabla_x\log p_t(x_t)\big]
= \nabla_x \log \big[ p_t(x_t)p_t(y\mid x_t)^w \big],
$
where $w>0$ is a guidance weight.
For generated clean samples, this tilts the generated distribution to $p(x)\exp(w \log p(y\mid x))$.
Intuitively, this generates samples that maximize the probability of having a specific value of $y$.
\citet{frans2025diffusion} showed that CFG can be used to achieve policy improvement in RL when $p_t(x_t)$ is the policy and the conditional variable $y$ is binary $y \in \{0, 1\}$ and $p(y=1\mid A(s, a))$ is a non-negative and monotonic increasing function of $A(s, a)$.

\section{Analysis of CFG for Advantage Conditioned Policies}
\label{sec:theoretical-issues-with-cfg}

Since CFG can induce policy improvement~\citep{frans2025diffusion}, one might wonder whether we can apply it to policies conditioned on continuous advantages to sample from our target, the KL-constrained optimal policy $\pi_\mathrm{KL}^*$ in \Cref{eqn:tilted-distribution}. In this section, we derive the solution of CFG on an advantage conditioned policy and show that it does not achieve this target.

Diffusion (and under some transformations flow matching) learn the score function, and so we can analyze the score function space directly. The score of the target distribution $\pi_\mathrm{KL}^*$ in \Cref{eqn:tilted-distribution} is
\begin{equation}
\label{eqn:score-of-tilt}
    \nabla_a \log \pi_\mathrm{KL}^*(a \mid s) 
    = \nabla_a \log \big[\hat{\pi}(a\mid s) \exp(A(s, a) / \beta) \big].
\end{equation}

Next, we will derive the score function when we perform CFG sampling on an advantage conditioned policy.
Let us denote the conditional policy as $\pi(a \mid s, y)$, where $y \sim Y$ is the conditional random variable. For simplicity, assume it has Gaussian likelihood: $Y \mid s, a \sim \mathcal{N}(A(s, a), \sigma_y^2)$ and $\log p_Y(y\mid A(s,a)) =-\frac{(y-A(s,a))^2}{2\sigma_y^2} -\frac12\log(2\pi\sigma_y^2)$.
We will denote the score of the conditional and unconditional policy as $g_y = \nabla_a \log \pi(a \mid s, y)$ and $g = \nabla_a \log \hat{\pi}(a \mid s)$.
Bayes rule gives:
$
g_y-g =\nabla_a \log \left[ \frac{p_Y(y\mid A(s, a)) \hat{\pi}(a\mid s)}{p(y\mid s)} \right] - \nabla_a \log \hat\pi(a\mid s)  
= \nabla_a \left[  -\frac{(y-A(s,a))^2}{2\sigma_y^2} \right].
$
Consequently, the score when sampling with CFG at guidance weight $w$ is
\begin{equation}
\label{eqn:cfg-induced-solution}
    g + w(g_y- g) 
    = \nabla_a \log \left[ \hat\pi(a\mid s) \exp\!\left(-\frac{w(y-A(s,a))^2}{2\sigma_y^2}\right) \right].
\end{equation}
This is generally not the same as the score of the desired tilted distribution in \Cref{eqn:score-of-tilt}. Intuitively, sampling with CFG induces a distribution that favors actions that match the target advantage value, and penalizes deviations above and below it; it does not always favor high-advantage actions as in \Cref{eqn:score-of-tilt}. Higher CFG weights more strongly favor matching the target advantage.

\section{Advantage Alignment Problem of Na\"ive Conditioned Policies}
\label{sec:problems}

While sampling with CFG does not necessarily produce actions that maximize the advantage, we might imagine that conditioning on a sufficiently high target advantage can still yield a good policy that is an improvement over $\hat{\pi}(a\mid s)$. 
By corollary, advantage conditioned policies should also enjoy this benefit because they equate to the special case of CFG sampling with guidance weight $w=1$. 
For simplicity, we focus our analysis on this special case, analyzing advantage-conditioned policies to understand whether they can effectively adapt to conditioning and generate high-advantage actions.

\begin{figure}[t]
    \centering
    \includegraphics[width=0.9\linewidth]{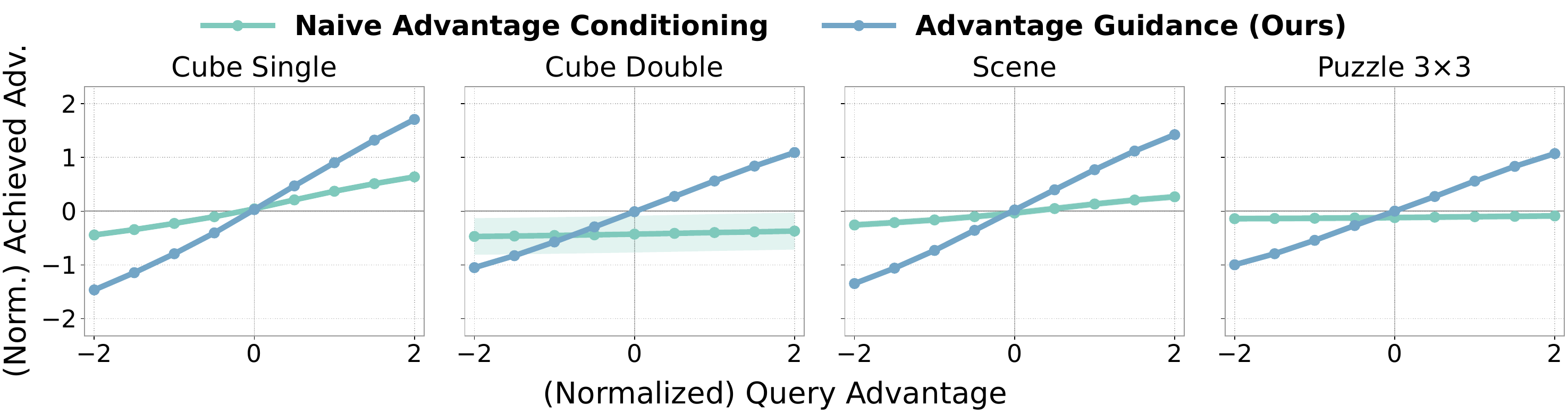}
    \caption{
    Na\"ive advantage conditioned policy suffers from the \textbf{advantage alignment problem}, generating actions whose advantages don't align with the requested value.
    $5$ tasks per domain.
    }
    \label{fig:not-pay-attention-to-advantage}
    \vspace{-1.5em}
\end{figure}
To measure how well a policy responds to the advantage signal, we define \textbf{\textit{advantage alignment}} as the extent to which the achieved advantage of a policy action matches the advantage used to query the conditional policy (\ie $A_\mathrm{query}$ vs. $A(s, \pi(\cdot\mid s, A_\mathrm{query}))$.
Intuitively, if a conditional policy responds properly to advantage inputs, it should be advantage aligned: actions generated conditioned on high advantage should have high advantages, and vice versa.
\Cref{fig:not-pay-attention-to-advantage} shows the advantage alignment on $20$ different OGBench~\citep{park2024ogbench} tasks across $4$ different domains, and the advantages shown are normalized using the mean and variance of the advantages of $32$ actions sampled from the same state.
A \naive advantage-conditioned policy $\pi(a\mid s, A)$, which we implement following RCP-A~\citep{kumar2019reward}, responds very weakly to advantage inputs, often generating actions that have the same advantages even when requested with different values, especially on hard long-horizon tasks.
When queried with a high advantage, it often generates actions whose advantages fall below the requested value.
We refer to this as the \textbf{\textit{advantage alignment problem}}.
And as we will show in \Cref{apdx:transformer-experiments}, using a more expressive policy architecture such as a transformer does not solve this problem, and misalignment is also reflected in policy performance when using different $A_\mathrm{query}$ (see \Cref{apdx:rcp-a-evals}).
In comparison, our method \methodname{} (which we will introduce later in \Cref{sec:method}) is substantially more responsive to advantage conditioning, with the resulting actions exhibiting achieved advantage values that closely align with requested advantages.

\begin{wrapfigure}{r}{0.5\linewidth}
    \centering
    \vspace{-1em}
    \includegraphics[width=\linewidth]{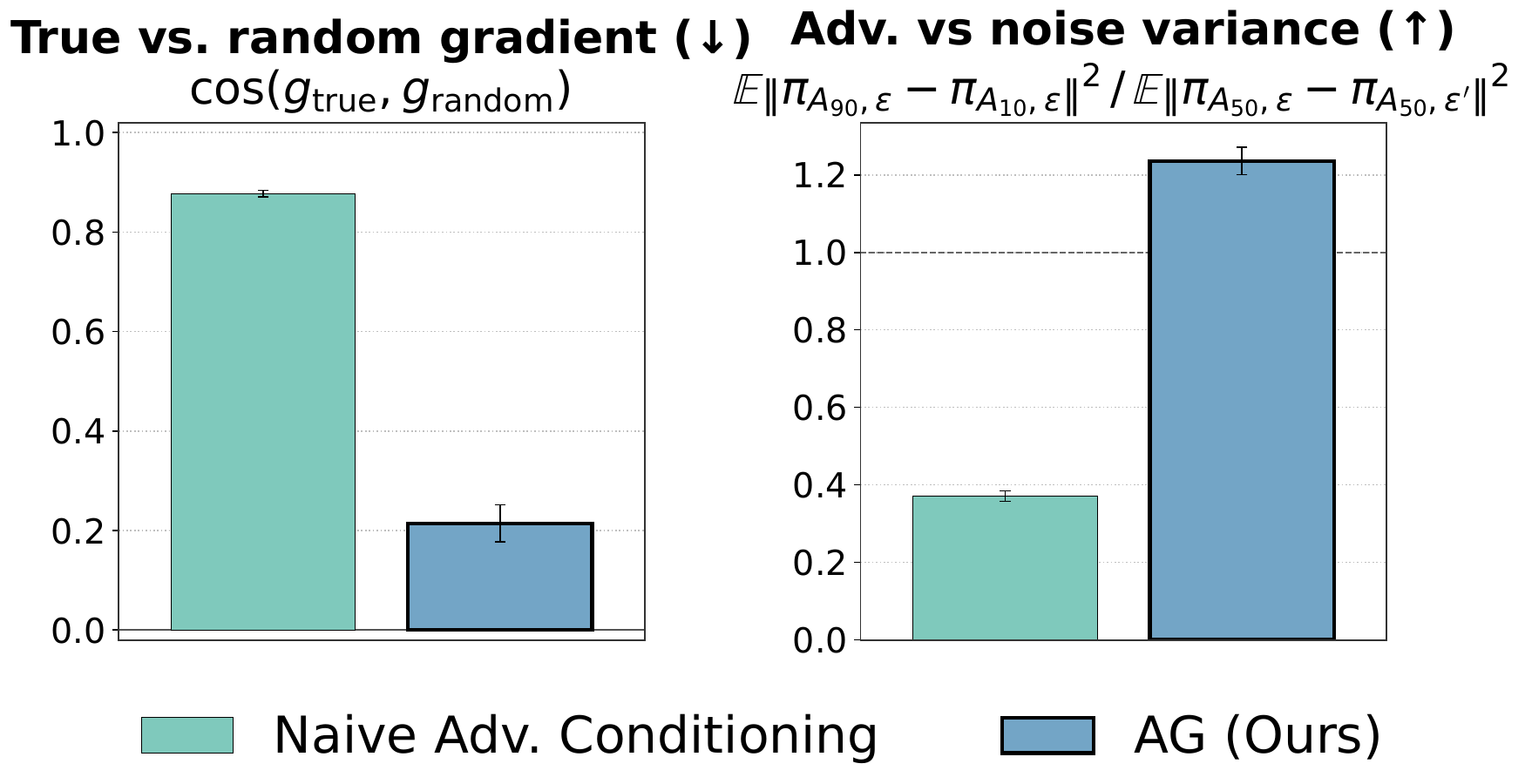}
    \caption{\textbf{Na\"ive advantage conditioned policies often ignore advantage inputs.} Averaged over $20$ OGBench tasks.}
    \label{fig:attention-challenges}
    \vspace{-1em}
\end{wrapfigure}
To gain a mechanistic understanding of the optimization process, we investigate (1) how the gradient of the policy loss responds to advantage conditioning, and (2) how strongly the output actions respond to different conditional signals.
First, we compare the gradient of the policy loss 
conditioned on the ground truth versus random advantage values. We plot their cosine similarity in \Cref{fig:attention-challenges} (left), and find that \naive advantage-conditioned policy has gradient cosine similarity close to $1$. This suggests that the loss is largely insensitive to advantage labels, leading the policy to ignore them.
So does the policy respond to any conditioning at all, or does it just ignore the advantage?
To investigate, we compare the sensitivity of generated actions to advantage inputs and initial flow noise. We measure $
\frac{\mathbb E\|\pi(a\mid s, A_{90}, \epsilon) - \pi(a\mid s, A_{10}, \epsilon)\|^2}
     {\mathbb E\|\pi(a\mid s, A_{50}, \epsilon)-\pi(a \mid, s, A_{50}, \epsilon')\|^2}, $
where $A_{n}$ is the $n$-th percentile advantage for state $s$ and $\epsilon, \epsilon'$ are different initial noises.
\Cref{fig:attention-challenges} (right) shows that varying the advantage changes the \naive agent's actions by less than half the effect of sampling noise, whereas our method's response exceeds it.

The analysis above shows that the policy network largely ignores the advantage value 
, effectively regressing to an unconditional policy $\pi(a\mid s)$. In the next section, we will present solutions to this challenge and show that our method can significantly improve advantage alignment.

\section{\methodname}
\label{sec:method}

In this section, we introduce our proposed method, \textcolor{darkerblue}{\methodname}, an advantage-conditioned policy extraction method that improves advantage alignment and can target the KL-constrained optimal policy through guidance. We first derive a guidance formulation we term \textit{contrastive guidance} that uses a policy conditioned on continuous advantage values to target $\pi_\mathrm{KL}^*(a\mid s) \propto \hat{\pi}(a\mid s) \exp(A(s, a) / \beta)$.
Then, we introduce practical techniques during training to improve advantage alignment, including reparameterizing the conditional policy as a residual flow over an unconditional flow field and training on multiple imagined actions per state.

\subsection{Contrastive Guidance for Advantage Maximization}
\label{sec:guidance-theory}
First, we aim to derive a sampling technique based on diffusion guidance that can enable the policy to generate advantage-maximizing actions rather than matching a prespecified value. Our target is the KL-regularized optimal policy $\pi_\mathrm{KL}^*$, which maximizes advantage under KL constraint to $\hat{\pi}(a\mid s)$.

Recall that for flow and diffusion policies, we can analyze the score function directly, and the score of our target distribution $\pi_\mathrm{KL}^*$ is in \Cref{eqn:score-of-tilt}. 
Following \Cref{sec:theoretical-issues-with-cfg}, we denote the conditional policy as $\pi(a\mid s, y)$ where $y \sim Y$ is the conditional random variable with likelihood $p_Y(y\mid A)$.
\begin{remark}[\textbf{Contrastive Guidance.}]
\label{thm:contrastive-guidance}
    Assume $p_Y$ belongs to exponential family $p_Y(y\mid A) = h(y) \exp\left(T(y)A - F(A)\right)$ where $h(y) \geq 0$ 
    and $F(A)$ is a log-partition function. 
    The action score of the KL-regularized optimal policy $\pi_\mathrm{KL}^*(a\mid s) \propto \hat{\pi}(a\mid s) \exp(A(s, a) / \beta)$ matches
    \begin{equation}
        \label{eqn:contrastive-guidance-score}
        \nabla_a \log \hat{\pi}(a\mid s) + w \left[ \nabla_a \log \pi(a\mid s, y_+) - \nabla_a \log \pi(a \mid s, y_-) \right]
    \end{equation}
    if some fixed $w, y_+, y_- \in \mathbb{R}$ satisfy $1/\beta = w [T(y_+) - T(y_-)]$.
\end{remark}
We provide a formalization and proof of this claim in \Cref{apdx:exponential-family-guidance}. 
In other words, if we sample according to \Cref{eqn:contrastive-guidance-score}, which is easily derived with the score of the conditional and unconditional policy, we can directly match the score of the desired exponentially-tilted distribution.
This is akin to performing diffusion guidance with guidance weight $w$ and guidance term $\nabla_a \log \pi(a\mid s, y_+) - \nabla_a \log \pi(a \mid s, y_-)$. 
Since the guidance term contains the score difference of a contrastive pair $(y_+, y_-)$, we call this \textit{contrastive guidance}.

\begin{remark}[\textbf{Advantage Maximization.}]
\label{rem:advantage-maximization}
If $\beta > 0$ (which is ensured if $T(y_+)>T(y_-)$ and $w>0$), the tilted policy
$\pi_\mathrm{KL}^*(a\mid s)\propto\hat{\pi}(a\mid s)\exp(A(s,a)/\beta)$
is the solution to the KL-constrained advantage maximization problem
$
\arg\max_{\pi(\cdot\mid s)}
\left\{
\mathbb{E}_{a\sim\pi(\cdot\mid s)}[A(s,a)]
-
\beta D_{\mathrm{KL}}\!\left(
\pi(\cdot\mid s)\,\|\,\hat{\pi}(\cdot\mid s)
\right)
\right\}.
$
\end{remark}
This follows directly from \Cref{thm:contrastive-guidance} and prior work~\citep{dayan1992feudal,abdolmaleki2018maximum,peng2019advantage}.
Thus, with a positive guidance weight $w>0$ and $T(y_+)>T(y_-)$, the target of contrastive guidance induces advantage maximization, rather than matching a fixed value. 
The contrastive labels $y_+, y_-$ and the guidance weight $w$ control the maximization degree together: a higher $w[T(y_+)-T(y_-)]$ leads to weaker constraint and stronger advantage maximization.

\paragraph{Practical Implementation.} In practice, we can choose a simple exponential family for the conditional variable $Y$, the Gaussian distribution. In other words, we choose $Y \mid s, a \sim \mathcal{N}(A(s, a), \sigma_y^2)$, which corresponds to
$
h(y)=\frac{1}{\sqrt{2\pi}\sigma_y}
\exp(-\frac{y^2}{2\sigma_y^2}) ,
T(y)=\frac{y}{\sigma_y^2},
F(A)=\frac{A^2}{2\sigma_y^2}.
$
For $y_+>y_-$, we automatically have $T(y_+)>T(y_-)$ and we can choose guidance weight 
$w=\frac{\sigma_y^2}{\beta(y_+-y_-)} > 0$ to give the desired tilt in \Cref{eqn:tilted-distribution}.

\subsection{Residual Parameterization for Conditional Policy}
\label{sec:residual-parameterization}
While contrastive guidance can target the KL-regularized optimal policy in theory, it makes an implicit assumption that the conditional policy is advantage-aligned with $y_+$ and $y_-$. Therefore, we focus next on training techniques that can improve advantage alignment.
Recall that we found in \Cref{sec:problems} that \naive advantage conditioned policies often ignore advantage inputs and largely optimize the policy loss based on state and initial noise alone. In other words, the policy is minimizing the loss for the advantage-conditional policy by learning to minimize the unconditional loss.
This motivates an approach that separates out the components that are not conditioned on advantage, and leaves the conditional policy to focus on the advantage input.
In \methodname{}, we propose to parameterize the conditional policy as a residual flow $v_\theta$ to the flow field of an unconditional reference policy $v_{\hat{\pi}}(a\mid s)$.
Specifically, both policies are instantiated as flow-matching networks, where the model learns a velocity field conditioned on a partially noised action $x_t$, along with the noise scale $t$, state $s$, and conditioning variable $y$. The residual parameterization in velocity space is:
\begin{equation}
\label{eqn:residual-policy-flow}
    v_\pi(x_t, t, s, y) = v_{\hat{\pi}}(x_t, t, s) + v_\theta(x_t, t, s, y).
\end{equation}
The training loss for the conditional policy is:
\begin{equation}
\label{eqn:training-loss}
    \mathcal{L}(\theta) = \mathbb{E}_{s, a \sim \mathcal{D}_{\mathrm{aug}}} \big[ || v_{\hat{\pi}}(s, x_t, t) + v_\theta (s, x_t, y, t) - (a - \epsilon)  ||^2 \big]  \quad \text{where} \;\; x_t = (1-t) \epsilon + t a.
\end{equation}
The noise scale $t \sim U[0, 1]$ and $\epsilon \sim \mathcal{N}(0, 1)$ is Gaussian noise.
The unconditional reference flow policy $\hat{\pi}$ is trained to minimize regular flow-matching loss $\mathbb{E}_{s, a \sim D}\left[|| v_{\hat{\pi}}(s, x_t, t) - (x - \epsilon)||^2\right]$.
With this parameterization, the reference policy $\hat{\pi}(a\mid s)$ minimizes the advantage-independent loss, therefore encouraging the residual $v_\theta$ to focus on minimizing $\mathcal{L}(\theta)$ based on conditional input $y$. 

\newcommand{\algcomment}[2][gray]{%
{\color{#1}\hfill$\blacktriangleleft$~\footnotesize\textit{#2}}}

\newcommand{\algcommentdown}[2][gray]{%
{\color{#1}$\blacktriangledown$~\footnotesize\textit{#2}}}

\begin{algorithm}[t]
\caption{\methodname{} (\methodshort{})}
\label{alg:cape}
\centering
\footnotesize
\begin{tabular}{
@{}p{0.6\linewidth}
@{\hspace{0.04\linewidth}}
p{0.35\linewidth}@{}
}
\textbf{Training} & \textbf{Sampling} \\
\midrule

\textbf{Input:} dataset $\mathcal D$, advantage $A(s,a)$,
$K$, advantage label std $\sigma_y$
&
\textbf{Input:} state $s$, trained velocities $v_{\hat\pi},v_\theta$ \\

\textbf{while} not converged \textbf{do}
&
\quad conditional query pair $y_+>y_-$,  \\

\quad $(s,a)\sim\mathcal D$, $\epsilon\sim\mathcal N(0,I)$,
$t\sim\mathcal U(0,1)$
&
\quad guidance weight $w$, steps $T$ \\

\quad Train ref. policy $\hat{\pi}$ with $\displaystyle
\mathcal L_{\mathrm{BC}}
=
\left\|v_{\hat\pi}-(a-\epsilon)\right\|_2^2
$
&
$x_0\sim\mathcal N(0,I)$ \\

\quad $a_0\gets a$; \;\;
$a_i\sim\hat\pi(\cdot\mid s)$, \;\; $i=1,\ldots,K$
\algcomment[ourblue]{action imagination}
&
\textbf{for} $n=0,\ldots,T-1$ \textbf{do} \\

\quad $A_i\gets A(s,a_i)$, \;\; $i=0,\ldots,K$
&
\quad $t\gets n/T$ \\

\quad $\tilde A_i\gets
(A_i-\operatorname{mean}(A_{0:K}))/
\sqrt{\operatorname{var}(A_{0:K})+\delta}$
\algcomment{adv. norm.}
&
\quad \algcommentdown[ourblue]{contrastive guidance} \\

\quad $y_i\sim\mathcal N(\tilde A_i,\sigma_y^2)$,
$i=0,\ldots,K$
\algcomment{Gaussian conditional var.}
&
\quad $v\gets v_{\hat\pi}+w\bigl[v_\theta(y_+)-v_\theta(y_-)\bigr]$ \\

\quad $\displaystyle
\Delta\theta \leftarrow \nabla_\theta
\sum_{i=0}^{K}
\left\|v_{\hat\pi}
+v_\theta(y_i)-(a_i-\epsilon)\right\|_2^2
$
\algcomment[ourblue]{residual param.}
&
\quad $x_{n+1}\gets x_n+\frac{1}{T}v$ \\

\textbf{end while}
&
\textbf{end for} \\

\textbf{return} $v_{\hat\pi},v_\theta$
& 
\textbf{return} $a=x_T$ \\

\bottomrule
\end{tabular}

\vspace{2pt}
\begin{minipage}{\linewidth}
\footnotesize
Velocity-field arguments are abbreviated above:
$v_{\hat\pi}=v_{\hat\pi}(x_t,t,s)$ and
$v_\theta(y)=v_\theta(x_t,t,s,y)$
\end{minipage}
\end{algorithm}

\subsection{Action Imagination for Counterfactual Supervision}
\label{sec:action-imagination}

Beyond optimization dynamics, learning an advantage-conditioned policy presents a more fundamental challenge: the dataset typically contains only one action for any given state, especially in continuous state spaces. As a result, the policy receives little direct supervision on how actions should change as requested advantages vary. 
We address this lack of counterfactual supervision with \textit{action imagination}: for each state-action pair $(s, a_d)$ in the dataset, we augment the training distribution for the conditional policy with $K$ additional actions on state $s$ sampled from some proposal distribution: 
$$
\mathcal D_{\mathrm{aug}}
=
\bigcup_{(s,a_d)\in\mathcal D}
\left\{
\bigl(s,a_i,A(s,a_i)\bigr)
\right\}_{i=0}^{K},
\qquad
a_0=a_d,\quad
a_i\sim\pi_p(\cdot\mid s)
\;\text{for }i=1,\ldots,K.
$$
One convenient solution is to set the proposal distribution as the reference unconditional policy: $\pi_p = \hat{\pi}$.
As we will show in \Cref{sec:experiments}, training with residual parameterization and action imagination drastically improves advantage alignment and policy performance, especially on long horizon tasks.

\subsection{Implementation Details}
\label{sec:implementation details}

\methodname{} is an RL policy extraction algorithm based on advantage-conditioned training and contrastive guidance during sampling. It is agnostic to how the advantage estimates are derived. For simplicity, in this paper, we train value functions $Q(s, a)$ and $V(s)$ using IQL~\citep{kostrikov2021offline}, and calculate advantage as $A(s, a) = Q(s, a) - V(s)$. 
\Cref{alg:cape} shows a summary of the training and sampling procedure of \methodname.
During training, we standardize advantages separately for each state
using $\tilde{A}(s,a_i) = (A(s,a_i)-\mu_s)/\sqrt{\sigma_s^2+\delta}$,
where $\mu_s = \frac{1}{K+1}\sum_{j=0}^{K} A(s,a_j)$ and
$\sigma_s^2 = \frac{1}{K+1}\sum_{j=0}^{K}(A(s,a_j)-\mu_s)^2$,
with dataset action $a_0=a_d$, imagined actions $a_{i=1,...,K}$, and $\delta>0$ ensuring numerical stability.
We find that this state-level normalization is more effective than batch-level and dataset-level normalization (see \Cref{apdx:additional-ablations}).
Then, we condition the policy on $\tilde{A}$ using AdaLN~\citep{xu2019understanding}, which adaptively scales and shifts hidden representations at each layer.

\section{Experiments}
\label{sec:experiments}

We conduct experiments to analyze the empirical effectiveness of \methodname{} on a range of long-horizon, sparse reward manipulation tasks.
Concretely, we study the following research questions:
\textbf{(1)} How well does \methodname{} perform compared to prior RL methods on datasets with varying degrees of optimality?
\textbf{(2)} Do residual parameterization and action imagination improve advantage alignment and policy performance?
\textbf{(3)} Does contrastive guidance improve policy performance compared to CFG or no guidance?
Additionally, we conduct further ablation experiments in \Cref{apdx:additional-ablations}. All detailed results are documented in \Cref{apdx:detailed-results}.

\paragraph{Experimental Setup.}
We evaluate all methods in the offline RL setting with $20$ tasks from OGBench~\citep{park2024ogbench} and $6$ tasks from OCBench~\citep{park2026behavioral}.
OGBench offers a suite of long-horizon challenging manipulation tasks, and OCBench is a recently proposed extension of OGBench that provides a more realistic action distribution meant to emulate the statistical properties of real-world human demonstrations.
In particular, the two benchmarks differ substantially in their offline dataset composition. 
OGBench is a standard offline RL benchmark with \textit{play-style} datasets generated by scripted behaviors pursuing random goals, often yielding highly suboptimal trajectories for the task objective.
OCBench offers \textit{demonstration-based} datasets, consisting of trajectories generated by near-expert policies with high success rates on the task objective. 
Compared to play-style datasets, these demonstrations cover narrower states and predominantly near-optimal actions, making it more challenging for offline RL because extrapolation errors are harder to avoid.
However, such data more closely resembles real-world robotics datasets, and OCBench’s design to emulate human demonstration statistics makes it particularly relevant.
We use \texttt{cube-single-T\{1-5\}}, \texttt{cube-double-T\{1-5\}}, \texttt{scene-T\{1-5\}}, \texttt{puzzle-3x3-T\{1-5\}} from OGBench and lite tasks \texttt{block-single-T1}, \texttt{block-double-T\{1-2\}}, \texttt{chamber-easy-T\{1-3\}} from OCBench, all with sparse rewards.
In addition, we follow the action chunking setting~\citep{li2025reinforcement}, where the higher-dimensional actions make the action distribution more complex and therefore more challenging.

\definecolor{adguide}{HTML}{73A5C6}
\definecolor{adguide}{HTML}{4F83A6}  %

\definecolor{edpcolor}{HTML}{D2D2D2}
\definecolor{fqlcolor}{HTML}{4D4D4D}  %
\definecolor{qgfcolor}{HTML}{DCEFE5}
\definecolor{cfgrlcolor}{HTML}{C1E1D1}
\definecolor{fawrcolor}{HTML}{A6D3BC}
\definecolor{rcpacolor}{HTML}{86C3A4}
\definecolor{rcprcolor}{HTML}{62AD8A}
\definecolor{crrcolor}{HTML}{3C9673}

\paragraph{Baselines.}
We compare \textcolor{adguide}{\textbf{\methodname{}}} against a representative set of strong flow-matching RL baselines, grouped into three categories:
\textcolor{black}{\textbf{(1) actor critic methods}} train a parameterized actor to maximize the critic $Q(s,a)$ subject to behavioral constraints. \textbf{FQL}~\citep{park2025flow} distills a multi-step flow policy into a single-step policy to avoid backpropagation through time, while \textbf{EDP}~\citep{kang2023efficient} approximates the multi-step integration with a single large Euler step.
\textcolor{crrcolor}{\textbf{(2) Supervised-learning-based methods}} formulate policy optimization as a conditional, filtered, or weighted supervised learning problem. \textbf{CRR}~\citep{wang2020critic} trains a behavior cloning policy on the subset of data with positive advantage; \textbf{FAWR} is an extension of AWR~\citep{peng2019advantage} to flow policies by weighting the flow-matching loss with exponentiated advantages. \textbf{RCP-R} and \textbf{RCP-A}~\citep{kumar2019reward} train a policy conditioned on the return-to-go and advantage estimate, respectively.
\textcolor{rcprcolor}{\textbf{(3) Diffusion guidance methods}} steer the denoising process toward high-value actions at inference time. \textbf{QGF}~\citep{zhou2026test} guides a behavior cloning policy using an approximation of the critic gradient $\nabla_a Q(s,a)$ at partially noised actions. \textbf{CFGRL}~\citep{frans2025diffusion} applies CFG between a policy trained on the full dataset and a policy trained on the positive-advantage subset.
\Cref{apdx:baselines} provides a more detailed description of each method. To ensure fair comparison among all the policy extraction methods, we use the same value function for all. For simplicity, use the IQL~\citep{kostrikov2021offline} $Q$ and $V$ functions, and use $A(s,a) = Q(s, a) - V(s)$.

\subsection{How Does \methodname{} Compare with Prior Methods?}
\label{sec:results-comparison-with-prior}

\begin{figure}[t]
    \centering
    \includegraphics[width=\linewidth]{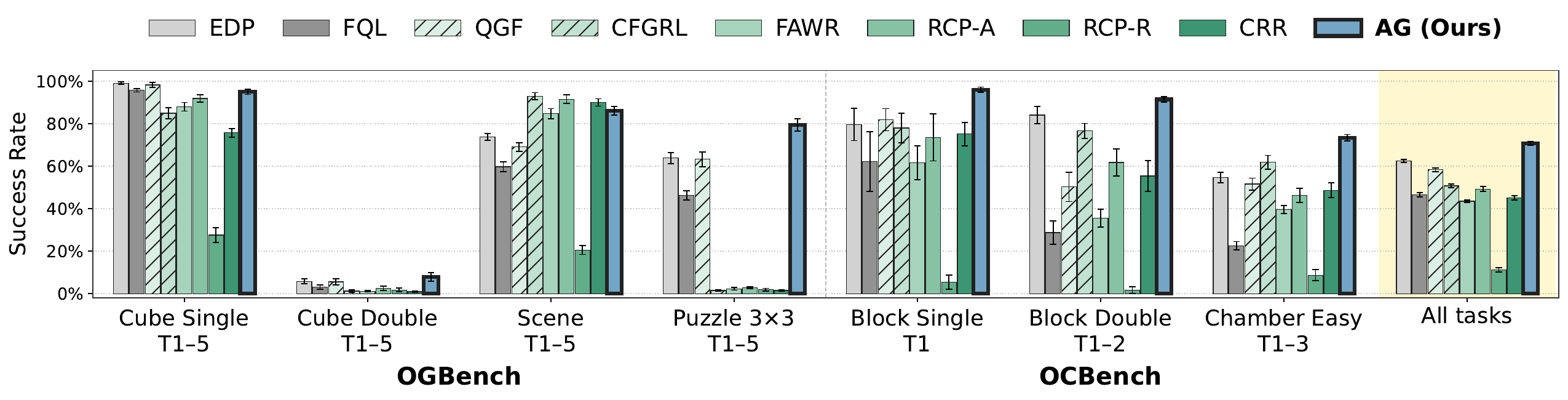}
    \caption{\textbf{Offline RL performance} on OGBench (20 tasks) and OCBench (6 tasks) with $10$ seeds: \methodname{} (\textcolor{adguide}{blue}) outperforms all previous methods, including strong actor-critic methods (\textcolor{fqlcolor}{grey}), supervised-learning-based methods (\textcolor{crrcolor}{green}), and diffusion guidance methods (\textcolor{rcprcolor}{shaded green}).}
    \label{fig:main-results}
    \vspace{-1.5em}
\end{figure}

While we expected the main benefit of \methodname{} to be its simplicity, in practice we found that it actually outperforms much more sophisticated RL methods. 
\Cref{fig:main-results} compares the performance of \methodname{} against baselines. Overall, \methodname{} significantly outperforms previous methods on both play-style datasets from OGBench and demonstration-based datasets from OCBench.
In particular, we find that previous supervised-learning-based methods perform poorly compared to strong actor critic methods such as EDP. The performance gap is especially large on challenging long horizon tasks such as \texttt{puzzle-3x3} and \texttt{block-double}.
In comparison, \methodname{} significantly outperforms these baselines, including strong actor critic methods such as EDP and FQL.
\methodname{} also significantly outperforms previous diffusion guidance methods: compared to CFGRL, \methodname{} can utilize continuous advantage signals for denser supervision, and can target the KL-regularized optimal policy with contrastive guidance; compared to QGF, \methodname{} doesn't rely on explicit $Q$-gradients in high-dimensional action spaces, which can be unstable or inaccurate.

\paragraph{Best-of-N sampling.}
\begin{wrapfigure}{r}{0.4\linewidth}
    \centering
    \vspace{-2em}
    \includegraphics[width=\linewidth]{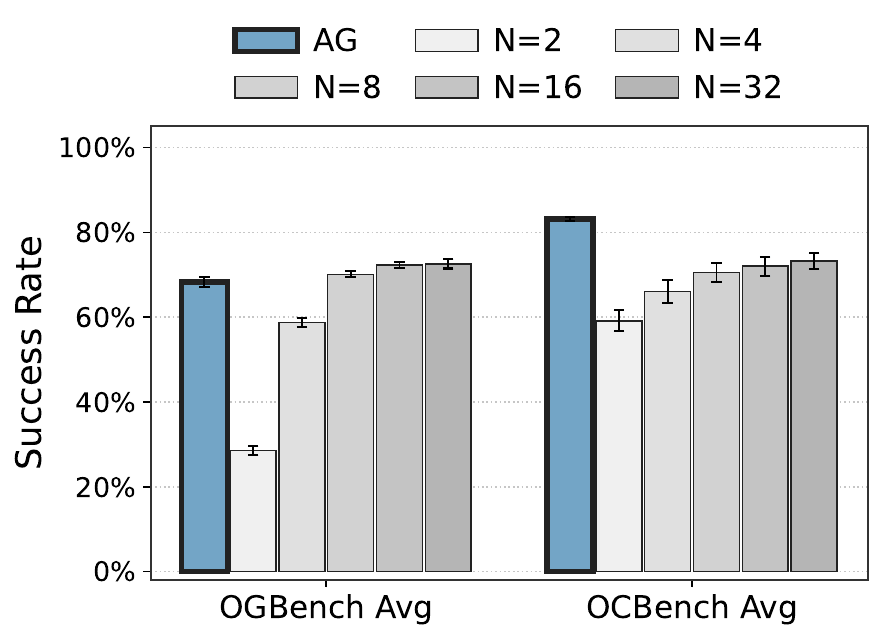}
    \caption{\textbf{Comparison with best-of-$N$ sampling}. \methodname{} outperforms BFN at small $N$ on OGBench and consistently on OCBench, despite using much less test-time compute.
    }
    \vspace{-0.9em}
    \label{fig:results-vs-bfn}
\end{wrapfigure}
Prior work has shown that when additional test-time compute is available, best-of-N (BFN) sampling is an expensive but effective method to improve flow policies~\citep{zhou2026test,hansen2023idql}. BFN samples multiple actions from a reference BC policy and chooses the action with the highest $Q$-value. However, even with a small number of samples such as $N=4$, BFN is much more expensive due to the sampling process (see \Cref{fig:bfn-compute-cost}). Regardless, we compare against this strong baseline in \Cref{fig:results-vs-bfn} with different values of $N$.
On OGBench, \methodname{} significantly outperforms BFN at low values of $N$ despite using less test-time compute, but eventually slightly underperforms BFN with large enough $N$, which also outperforms all prior methods in \Cref{fig:main-results}.
On OCBench, however, \methodname{} is significantly better even with large $N$.
This suggests that while BFN can be quite effective on play-style datasets, it is less competitive on more realistic demonstration-based datasets, even with sufficiently large $N$.

\begin{figure}[b]
    \vspace{-0.5em}
    \centering
    \begin{minipage}{0.55\linewidth}
        \centering
        \includegraphics[width=\linewidth]{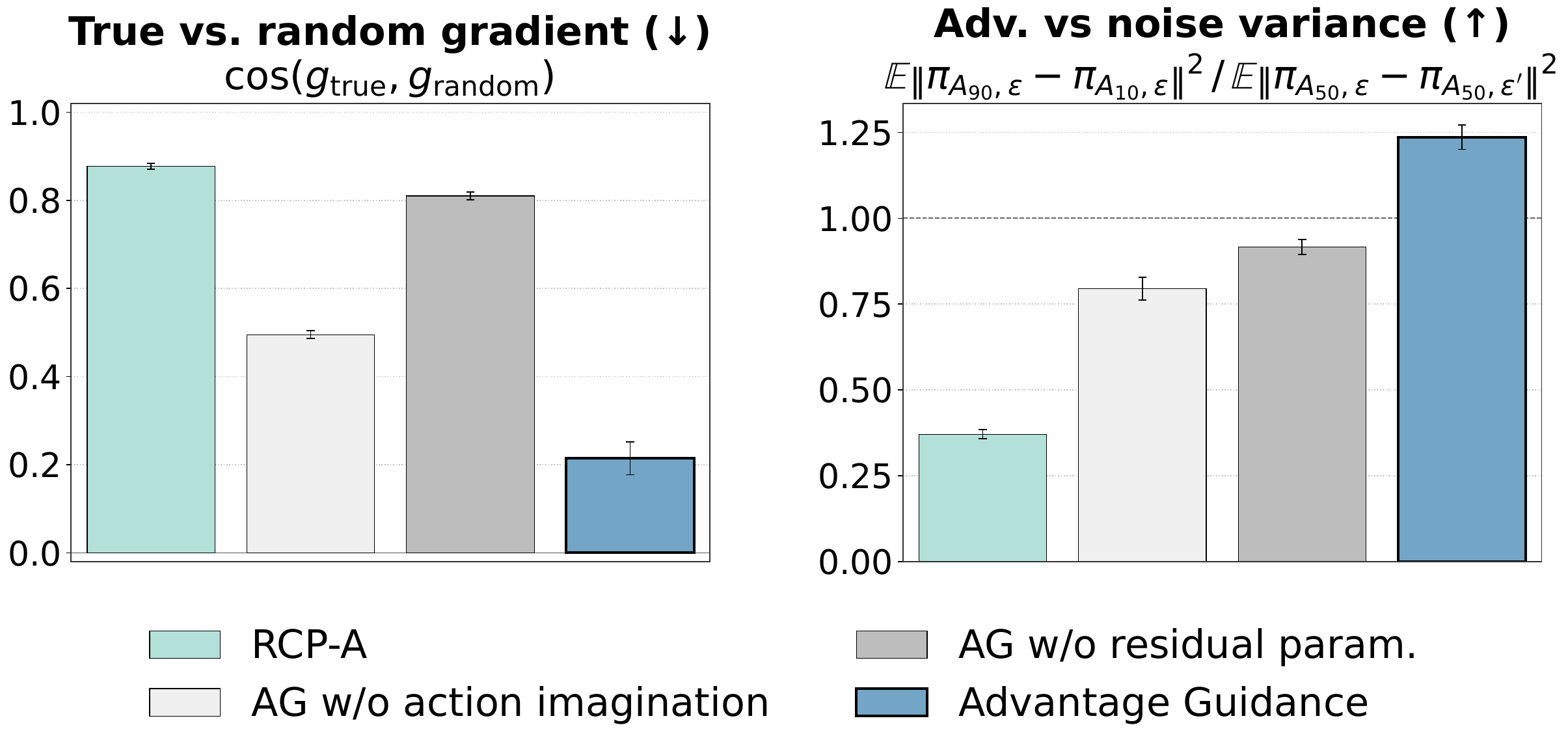}
        \caption{\textbf{Effect of residual parameterization and action imagination on advantage alignment}: both improve gradient sensitivity and output variance on $20$ OGBench tasks. 
        }
        \label{fig:ablate-residual-aug-metrics}
    \end{minipage}
    \hfill
    \begin{minipage}{0.4\linewidth}
        \centering
        \includegraphics[width=\linewidth]{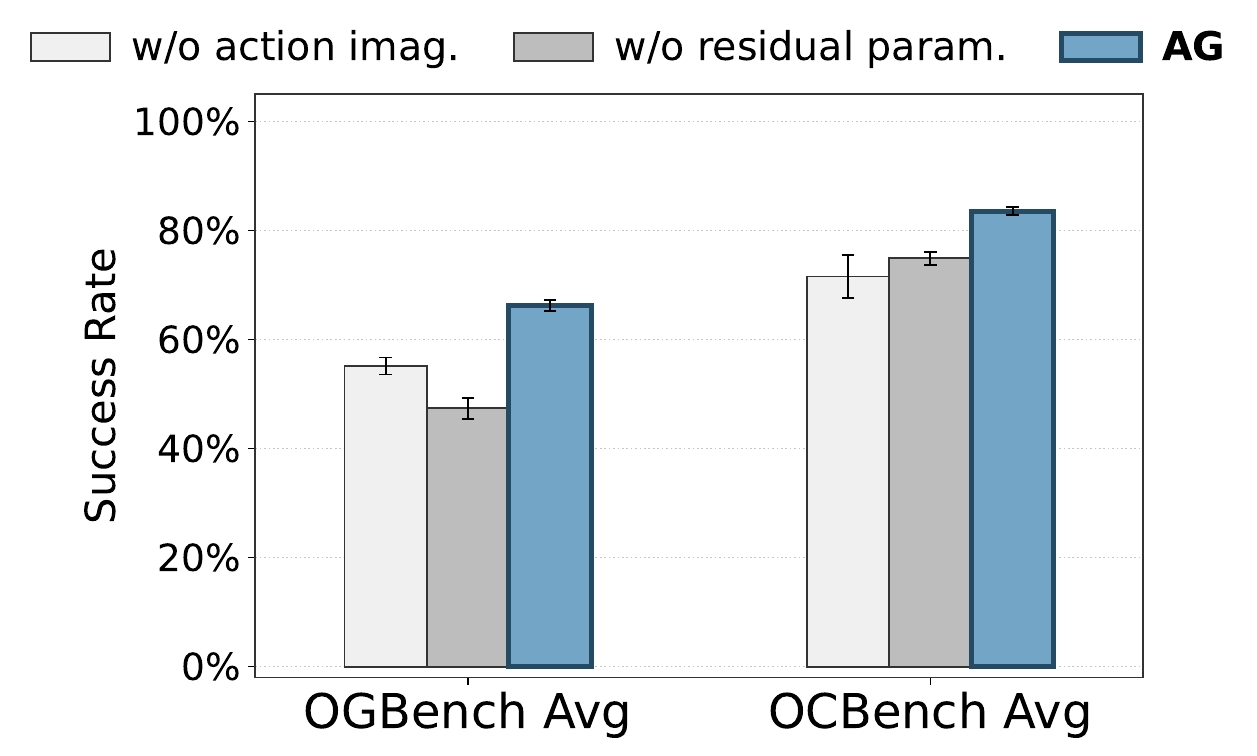}
        \vspace{-0.5cm}
        \caption{\textbf{Ablation of residual parameterization and action imagination training}: both improve policy performance across OGBench and OCBench.
        }
        \label{fig:ablate-residual-aug-performance}
    \end{minipage}
    \vspace{-0.4cm}
\end{figure}

\subsection{Does AG Training Improve Advantage Alignment and Policy Performance?}
\label{sec:results-ablate-residual-param-action-imagine}

We analyze whether the residual parameterization and the action imagination training scheme introduced in \Cref{sec:method} improve advantage alignment and policy performance.
We first plot in \Cref{fig:ablate-residual-aug-metrics} the two metrics introduced in \Cref{sec:problems} that measure how much a conditional policy responds to advantage conditioning: the gradient similarity between true and random gradients and the advantage-noise variance ratio.
\Cref{fig:ablate-residual-aug-metrics} shows that both the residual parameterization and action imagination improve advantage alignment, and combining both does even better.
\Cref{fig:ablate-residual-aug-performance} shows that such increased alignment translates to better policy performance on both OGBench and OCBench.
We also ablate the effect of different numbers of imagined actions in \Cref{fig:ablation-different-k} in \Cref{apdx:additional-ablations}, showing that increasing $K$ can increase performance but will plateau after $K=7$.

\subsection{Does contrastive guidance improve performance over CFG or no guidance?}
\label{sec:results-contrastive-guidance}

\begin{wrapfigure}{r}{0.4\textwidth}
    \centering
    \includegraphics[width=\linewidth]{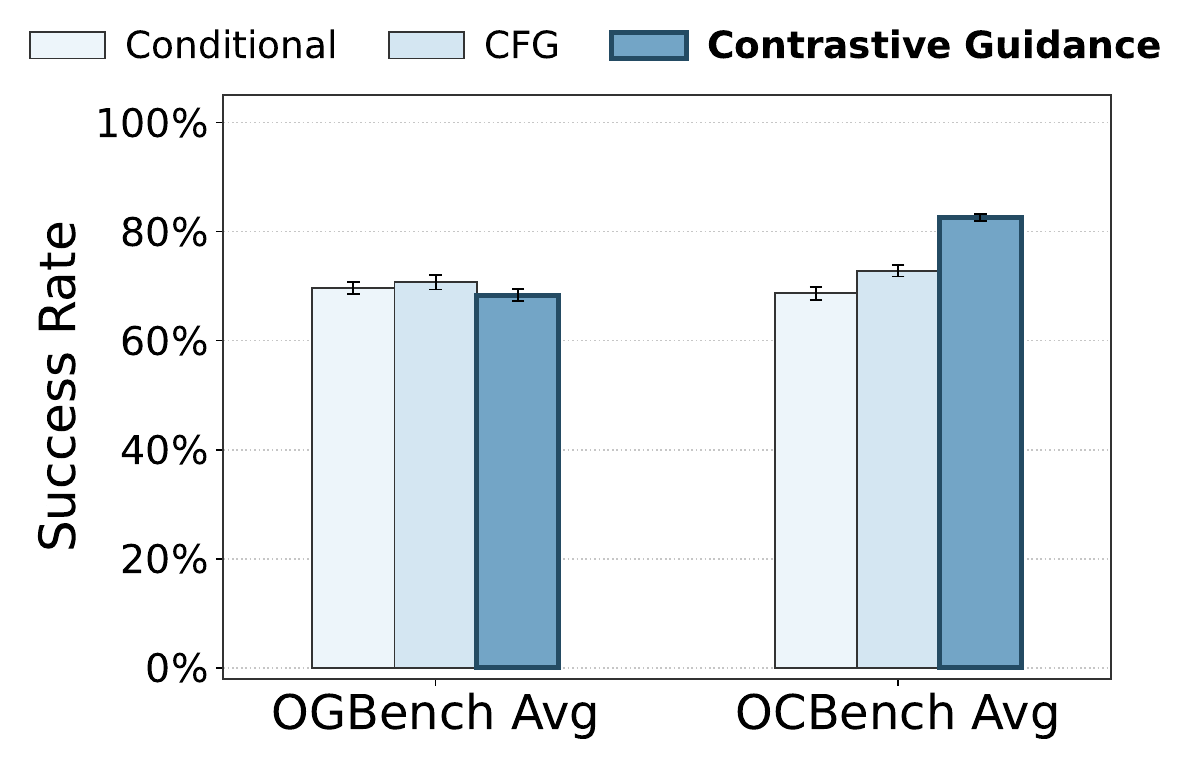}
    \caption{\textbf{Ablation of sampling method for \methodname{}.} Contrastive guidance significantly outperforms CFG and no guidance on $6$ OCBench tasks.}
    \label{fig:no-cfg}
\end{wrapfigure}
We showed in \Cref{sec:guidance-theory} that contrastive guidance can target the KL-regularized optimal policy and always favor high-advantage actions. Here, we empirically compare contrastive guidance with other sampling methods.
Specifically, we compare performance of the same \methodshort{}-trained policy when sampled with (1) contrastive guidance, (2) CFG, and (3) just querying the conditional policy.
For contrastive guidance, we use fixed $(y_+, y_-)$ per benchmark, while we tune the best conditional values per-domain per benchmark for the latter two to make them as competitive as possible.
\Cref{fig:no-cfg} shows that contrastive guidance performs similarly to CFG and pure conditional policy on OGBench, but significantly improves performance on OCBench.
It is not entirely surprising that all three methods can perform well: while CFG and pure conditional policy can only match a target advantage value instead of maximizing advantage, conditioning an advantage-aligned policy on a high-enough advantage should have a similar maximizing effect.
However, as observed on OCBench, conditional policies and standard CFG may perform poorly when high query advantages lie outside the training distribution, a common issue in demonstration-based datasets. In this regime, contrastive guidance is crucial for advantage maximization. See \Cref{apdx:conditional-sampling} for more discussions.
Furthermore, all methods only perform well if the conditional policy is advantage aligned (such as \methodshort{} here); using an unaligned policy such as RCP-A leads to substantially worse performance as \Cref{fig:main-results} shows.

\section{Discussions}
\label{sec:discussions}

We demonstrated in this paper that \methodname{} can outperform more complex RL methods with a simple supervised learning objective with advantage conditioning and contrastive diffusion guidance during inference. Our work serves as a step towards simple, stable, yet performant off-policy RL algorithms. There are several promising directions for future research. First, our implementation uses a straightforward advantage estimate. While we find this to be empirically effective, a natural next step is to explore whether more intricate advantage estimates improve policy performance, and how policy extraction can be robust in case of noisy advantage estimates. Second, our approach uses contrastive guidance to maximize advantage during inference time, but too large a guidance weight usually lead to a decline in performance. Developing techniques to stably utilize larger guidance weights can further improve the performance of \methodname{}. Finally, it would be interesting to explore \methodname{} in (offline-to-)online RL settings.

\section{Acknowledgments}
This research was partly supported by ONR N00014-25-1-2060 and AFOSR FA9550-22-1-0273. 
We would like to thank Andy Peng and Charles Xu for helpful discussions and early ideation, and Seohong Park for feedback on early drafts of the paper.
This research used the Savio computational cluster resource provided by the Berkeley Research Computing program at UC Berkeley.

\bibliography{ref}
\bibliographystyle{plainnat}

\appendix
\section{Appendix}

\subsection{Contrastive Guidance Theoretical Guarantee.}
\label{apdx:exponential-family-guidance}

\begin{theorem}[Contrastive guidance]
\label{thm:formal-contrastive-guidance}
Let $\hat{\pi}$ be a reference policy and $\pi(a\mid s, y)$ be the conditional policy, where the conditional random variable $Y$ has likelihood
\begin{align}
p_Y(y\mid A)
=
h(y)\exp\left(T(y)A-F(A)\right),
\qquad
F(A)=\log\int h(y)e^{T(y)A}\,dy,
\end{align}
where $h(y)\geq 0$.
For fixed $w,y_+,y_-$ satisfying
$0<h(y_\pm)<\infty$, $|T(y_\pm)|<\infty$, the policy
$\pi_{\mathrm{KL}}^*(a\mid s)
\propto\hat{\pi}(a\mid s)e^{A(s,a)/\beta}$
satisfies
\begin{align}
    \nabla_a\log\pi_{\mathrm{KL}}^*(a\mid s)
    &=
    \nabla_a\log\hat{\pi}(a\mid s)
    +w\left[
    \nabla_a\log\pi(a\mid s,y_+)
    -\nabla_a\log\pi(a\mid s,y_-)
    \right].
\end{align}
\end{theorem}

\begin{proof}
By Bayes' rule, we have
\begin{align}
    \pi(a\mid s,y)
    =
    \frac{\hat{\pi}(a\mid s)p_Y(y\mid A(s,a))}
    {p_Y(y\mid s)},
\end{align}
where
$p_Y(y\mid s)
=\int_{\mathcal{A}}\hat{\pi}(a'\mid s)p_Y(y\mid A(s,a'))\,da'$
is the marginal distribution of $Y$.
Taking the ratio of the two conditional policies gives
\begin{align}
    \log\frac{\pi(a\mid s,y_+)}{\pi(a\mid s,y_-)}
    &=
    \log\frac{p_Y(y_+\mid A(s,a))}
    {p_Y(y_-\mid A(s,a))}
    + \log\frac{p_Y(y_-\mid s)}{p_Y(y_+\mid s)}
    \\
    &=
    \left[T(y_+)-T(y_-)\right]A(s,a)
    + \log\frac{h(y_+)}{h(y_-)}
    + \log\frac{p_Y(y_-\mid s)}{p_Y(y_+\mid s)},
\end{align}
where the log-normalizer $F(A(s,a))$ cancels out.
Since the last two terms are independent of $a$, differentiating yields
\begin{align}
    \nabla_a\log\pi(a\mid s,y_+)
    - \nabla_a\log\pi(a\mid s,y_-)
    =
    \left[T(y_+)-T(y_-)\right]\nabla_a A(s,a).
\end{align}
Using $\beta^{-1}=w[T(y_+)-T(y_-)]$, we therefore obtain
\begin{align}
    &\nabla_a\log\hat{\pi}(a\mid s)
    + w\left[
    \nabla_a\log\pi(a\mid s,y_+)
    - \nabla_a\log\pi(a\mid s,y_-)
    \right]
    \\
    &\qquad=
    \nabla_a\log\hat{\pi}(a\mid s)
    + \frac{1}{\beta}\nabla_a A(s,a)
    \\
    &\qquad=
    \nabla_a\log
    \left(
    \frac{\hat{\pi}(a\mid s)e^{A(s,a)/\beta}}{Z_\beta(s)}
    \right)
    \\
    &\qquad=
    \nabla_a\log\pi_{\mathrm{KL}}^*(a\mid s),
\end{align}
\end{proof}

We note that the result above characterizes the desired policy through its clean-action score. For sampling with contrastive guidance, we extend it to intermediate denoising steps as a practical approximation, following standard classifier free guidance~\citep{ho2022classifier}.

\subsection{Additional \methodname{} Ablation Experiments}
\label{apdx:additional-ablations}

\paragraph{Action imagination.}
\Cref{fig:ablation-different-k} shows the ablation of using different numbers of imagined actions for \methodname{}. For each, the conditional policy is trained on $K+1$ actions per state, with $K$ imagined actions and $1$ dataset action. As described in \Cref{sec:implementation details}, the advantages are also normalized with the mean and variance of these $K+1$ actions. \Cref{fig:ablation-different-k} shows that having a very small $K$ such as $K=3$ hurts performance, and having more imagined actions helps, though performance plateaus after roughly $K=7$.

\begin{figure}[h]
    \centering
    \includegraphics[width=0.8\linewidth]{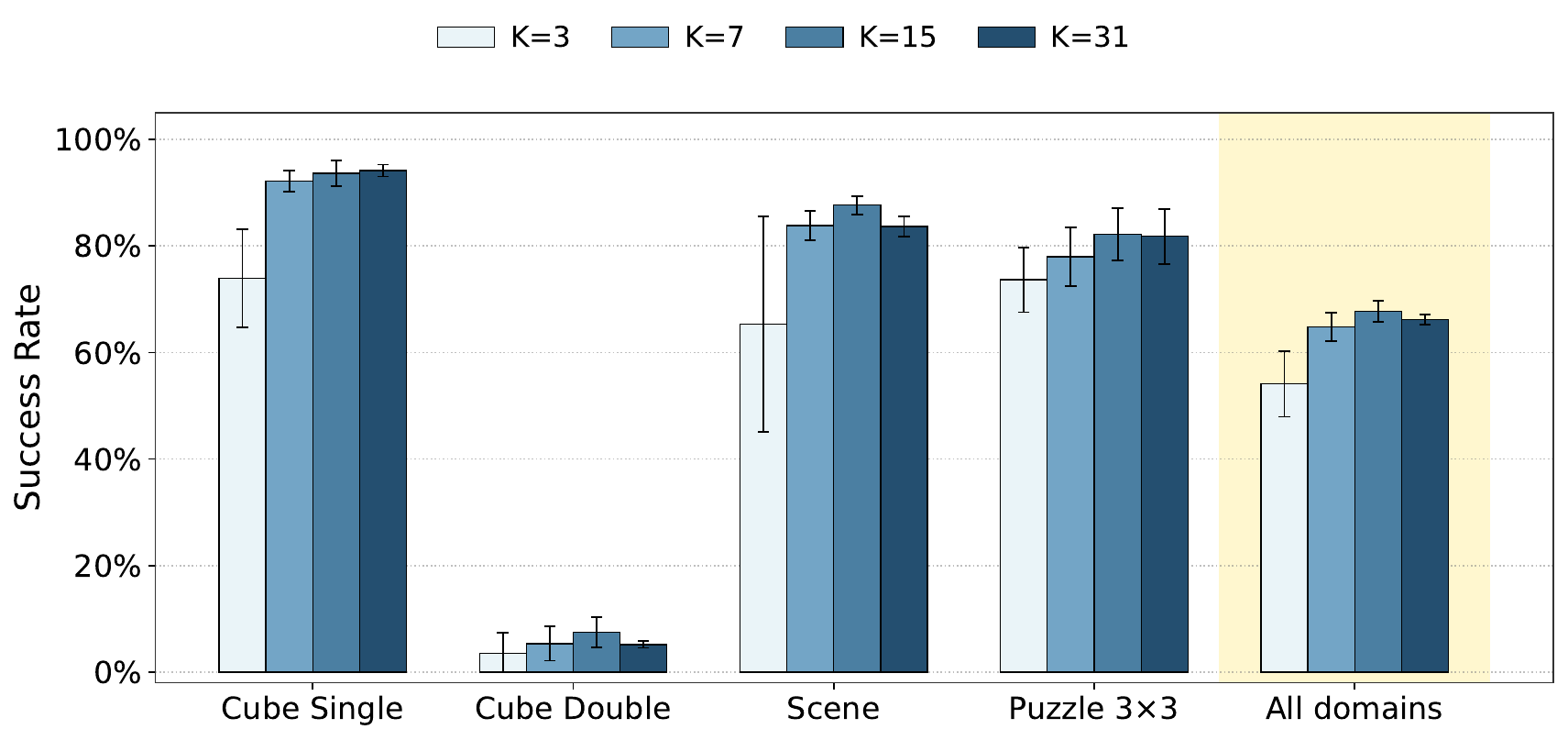}
    \caption{Ablating the number of imagined actions, $K$, used for training the conditional policy in \methodname{}. 
    A small $K$ underperforms larger values, but performance quickly plateaus.
    Averaged over $5$ tasks per domain and $4$ seeds.
    }
    \label{fig:ablation-different-k}
\end{figure}

\paragraph{Advantage Normalization.}
\Cref{fig:ablation-different-normalization} shows the ablation of using different methods of advantage normalization. We consider three different types of normalization methods that normalize with the mean and variance with $\tilde{A}(s,a_i) = (A(s,a_i)-\mu)/\sqrt{\sigma^2+\delta}$, and each method calculates the mean and variance statistics differently:
\begin{enumerate}
    \item Per-state norm (Ours): as described in \Cref{sec:implementation details}, we calculate mean and variance per-state using $K$ imagined action together with $1$ dataset action on the same state:
    \[
\mu_s = \frac{1}{K+1}\sum_{j=0}^{K} A(s,a_j),
\qquad
\sigma_s^2 = \frac{1}{K+1}\sum_{j=0}^{K}\left(A(s,a_j)-\mu_s\right)^2,
\]
where $a_0=a_d$ is the dataset action and $a_1,\ldots,a_K$ are imagined actions.
    \item Per-batch norm: we calculate the mean and variance across the current minibatch of advantages \[
\mu_{\mathcal B} = \frac{1}{B}\sum_{i=1}^{B} A(s_i,a_i),
\qquad
\sigma_{\mathcal B}^2 = \frac{1}{B}\sum_{i=1}^{B}
\left(A(s_i,a_i)-\mu_{\mathcal B}\right)^2.
\]

    \item Dataset-level norm: we calculate the mean and variance across the entire dataset: \[
\mu_{\mathcal D} = \frac{1}{N}\sum_{i=1}^{N} A(s_i,a_i),
\qquad
\sigma_{\mathcal D}^2 = \frac{1}{N}\sum_{i=1}^{N}
\left(A(s_i,a_i)-\mu_{\mathcal D}\right)^2.
\]
\end{enumerate}
To only ablate the effect of normalization, the different normalization methods all train the conditional policy with action augmentation at $K=31$ and only differ in how the advantages are normalized.
\Cref{fig:ablation-different-normalization} shows that our per-state normalization scheme offers a big advantage over other normalization schemes.

\begin{figure}[h]
    \centering
    \includegraphics[width=0.8\linewidth]{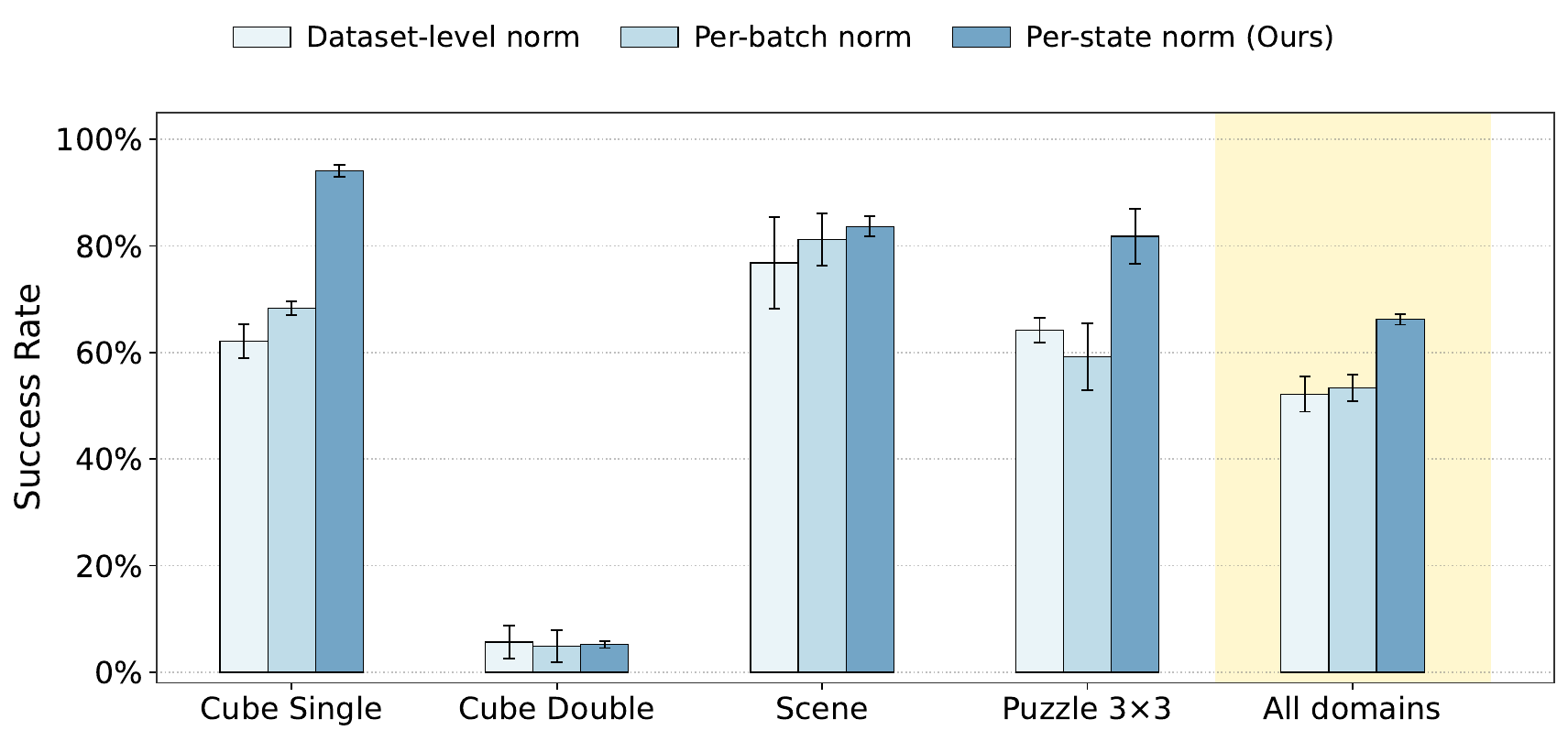}
    \caption{
    Ablating different advantage normalization schemes: our per-state advantage normalization outperforms batch-level and dataset-level normalization. Averaged over $5$ tasks per domain and $4$ seeds.}
    \label{fig:ablation-different-normalization}
\end{figure}

\paragraph{Contrastive guidance conditions and guidance weight.}
Additionally, we measure how the contrastive conditions $(y_+, y_-)$ pair and the guidance weight $w$ affect policy performance. \Cref{fig:dense-condition-pair} shows the performance with $6$ different contrastive pairs at each guidance weight and at the best guidance weight. With the best guidance weight $w$, different contrastive conditions perform similarly. This aligns with our derivation in \Cref{sec:guidance-theory}, which shows that the optimality of the sampled action is determined by the product $w(y_+ - y_-)$, and so in practice it's sufficient to tune only one.

\begin{figure}[h]
    \centering
    \includegraphics[width=0.8\linewidth]{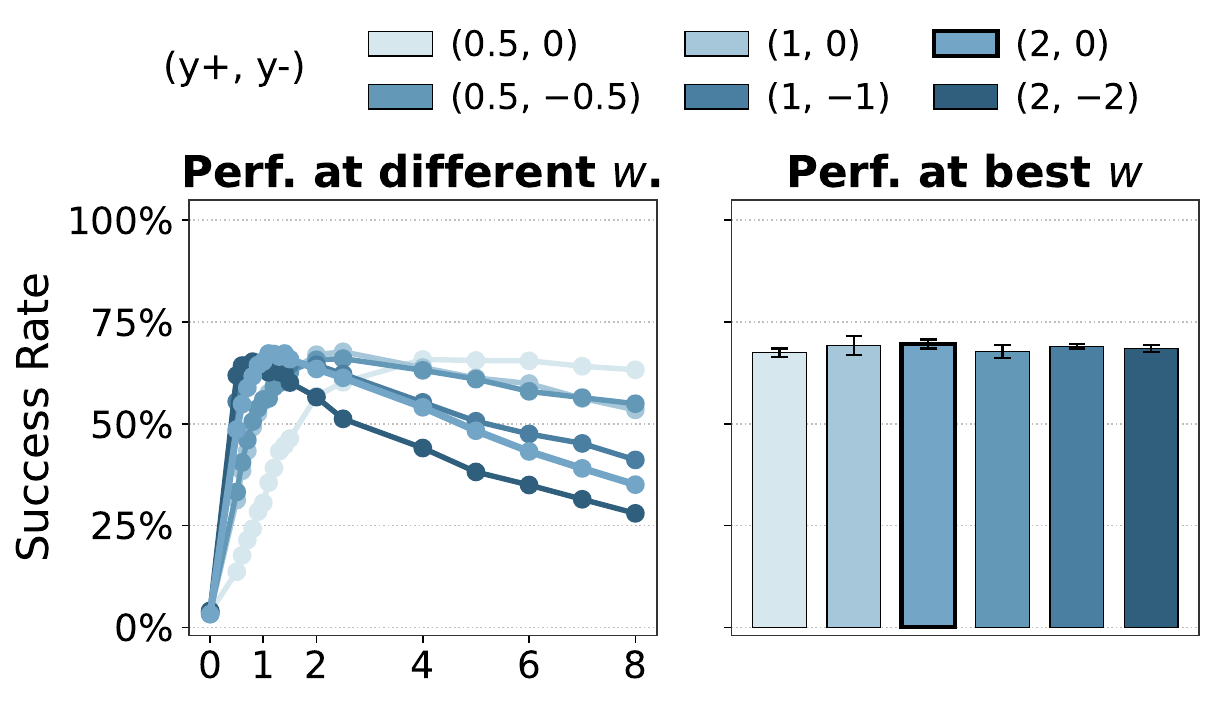}
    \caption{\textbf{Sensitivity to contrastive conditions and guidance weight:} each contrastive condition performs similarly given the best guidance weight for that condition. }
    \label{fig:dense-condition-pair}
\end{figure}

\paragraph{Advantage label standard deviation.}
\Cref{fig:ablation-different-label-noise} ablates different values of standard deviation we use to sample the advantage label $Y\mid s, a \sim N(A(s,a), \sigma_y^2)$. 
To plot the performance, we use contrastive pair $(y_+, y_-) = (2, 0)$ for all variations, and pick the best guidance weight for each variation on each domain according to our hyperparameter selection scheme in \Cref{apdx:hyperparameter}.
The results show that different values of $\sigma_y^2$ result in roughly similar levels of performance.
This is perhaps not so surprising given that we showed in \Cref{sec:guidance-theory} the label standard deviation, guidance weight, and contrastive pair together determine the advantage tilt: $\beta = \frac{\sigma_y^2}{w (y_+ - y_-)}$. 
As a result, using different $\sigma_y$ does not have a significant effect on performance as long as the guidance weight $w$ is changed accordingly as well. This is similar to the finding in \Cref{fig:dense-condition-pair}.
We also note that while our derivation in \Cref{sec:guidance-theory} require strictly $\sigma_y > 0$, $\sigma_y=0$ works well in practice.

\begin{figure}[h]
    \centering
    \includegraphics[width=0.8\linewidth]{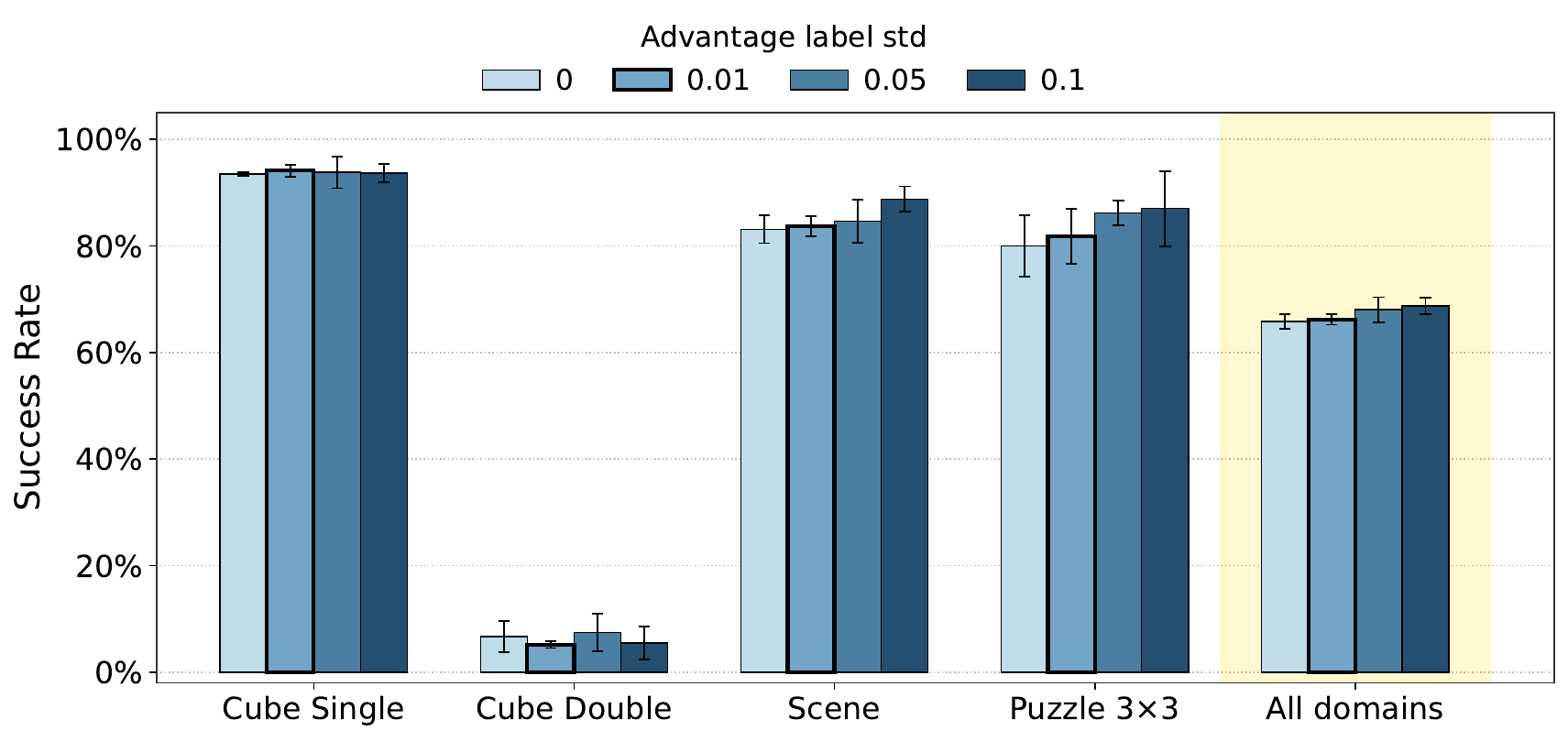}
    \caption{Sensitivity to different advantage label standard deviation: performance is not sensitive to the label standard deviations as long as the best guidance weight is picked. Averaged over $5$ tasks per domain and $4$ seeds.}
    \label{fig:ablation-different-label-noise}
\end{figure}

\subsection{Empirical Evaluation of Na\"ive Advantage Conditioned Policy RCP-A}
\label{apdx:rcp-a-evals}

\Cref{sec:problems} showed that a \naive advantage conditioned policy such as RCP-A suffers from the advantage alignment problem. We empirically evaluate that here, evaluating RCP-A on $20$ OGBench tasks when queried with different advantage values. As described in \Cref{apdx:baselines}, RCP-A models the distribution of advantage value over the dataset by modeling it as a Gaussian distribution and learning its mean $\mu$ and variance $\sigma^2$. During inference, RCP-A queries with advantage $\mu + k \sigma$, which is $k$ standard deviations above the mean dataset advantage. While the original paper~\citep{kumar2019reward} always set $k=1$, we experiment with different $k$ values during inference and plot the result in \Cref{fig:rcp_a_conditional_evals}. The results show that different query values has effectively no impact on performance, which corresponds with our empirical analysis in \Cref{sec:problems}.

\begin{figure}[h]
    \centering
    \includegraphics[width=0.8\linewidth]{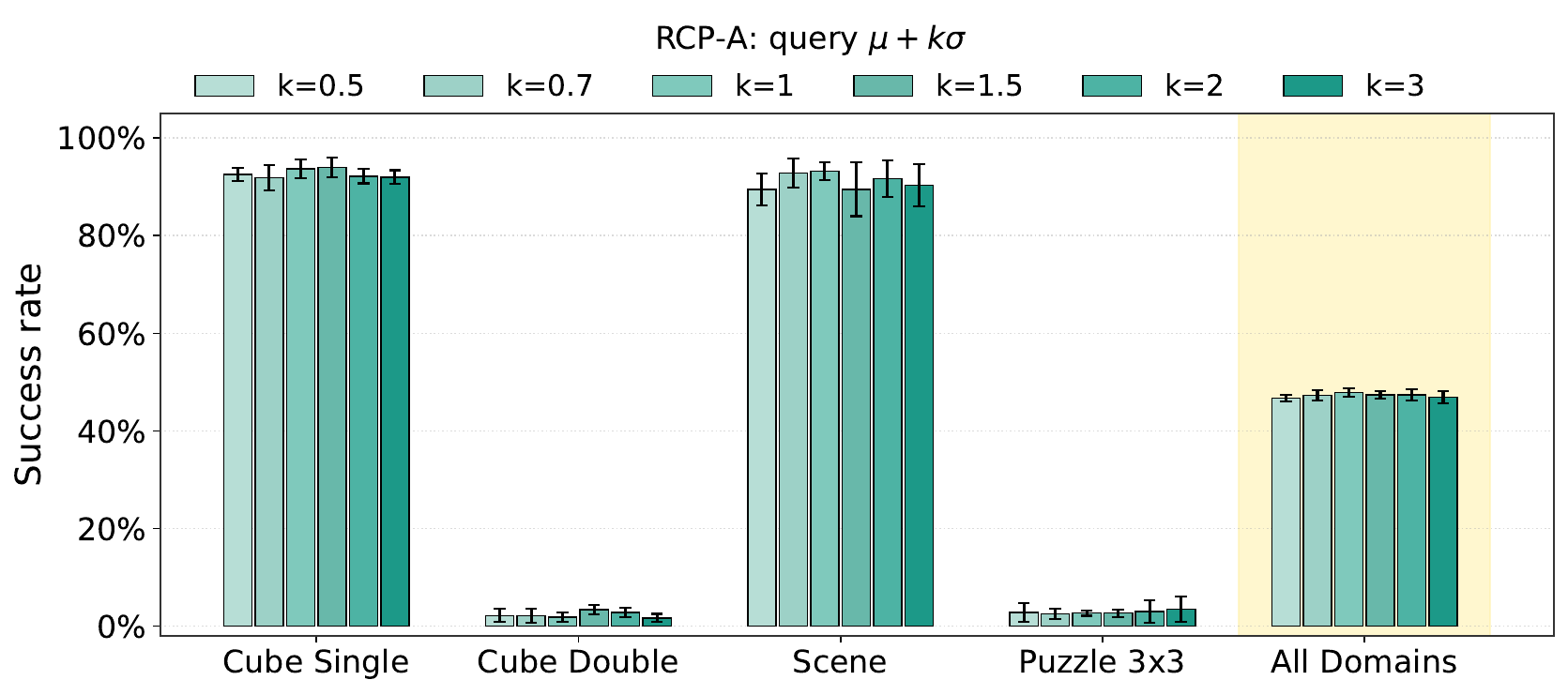}
    \caption{Na\"ive advantage conditioned policy RCP-A has the same performance across $20$ OGBench tasks ($5$ tasks per domain) when queried with different advantage values. Averaged over $4$ seeds.}
    \label{fig:rcp_a_conditional_evals}
\end{figure}

\subsection{Does A More Expressive Architecture Help with Advantage Alignment?}
\label{apdx:transformer-experiments}

In this section, we investigate whether we can improve advantage alignment, and in turn policy performance, of \naive advantage conditioning by using a more expressive conditional policy architecture.
Specifically, we experiment with using a transformer architecture to ingest the conditional advantage.
We compare RCP-A, which uses a FiLM-like architecture that uses the advantage input to modulate the hidden states at every layer of the policy network, to a variant that instead uses a transformer to handle conditional inputs.
We represent each action position in the \(H=5\) action chunk with a separate token. These action tokens, together with an advantage token, an observation token, and a flow-time token, are processed jointly through self-attention. Each contextualized action token is then projected directly to the velocity vector for its corresponding action position. We use four transformer blocks with model width 256, eight attention heads (32 dimensions per head), and feed-forward width 1024. The FiLM RCP-A actor has approximately \(1.39\)M parameters, while the transformer actor has approximately \(3.23\)M parameters.
We tune the learning rate smaller ($1e-4$) for the transformer version to prevent overfitting.

We plot the performance comparison between RCP-A and its transformer variant, as well as \methodname{}, in \Cref{fig:rcp-transformer}.
\Cref{fig:rcp-transformer} (left) shows that using a transformer is not significantly better in terms of advantage alignment: while the transformer policy output varies more when advantage input is changed, the policy gradient is still invariant to advantage inputs. In comparison, \methodname{} is much more responsive to different advantage inputs.
\Cref{fig:rcp-transformer} (right) plots the policy performance on OGBench. Overall, using a transformer does not affect policy performance much, and even slightly underperforms the original RCP-A agent.
This shows that there is no free lunch: we cannot expect to improve advantage alignment by simply swapping in a transformer policy backBFNe, and the residual parameterization and action imagination training schemes are critical.

\begin{figure}[h]
    \centering
    \includegraphics[width=0.48\linewidth]{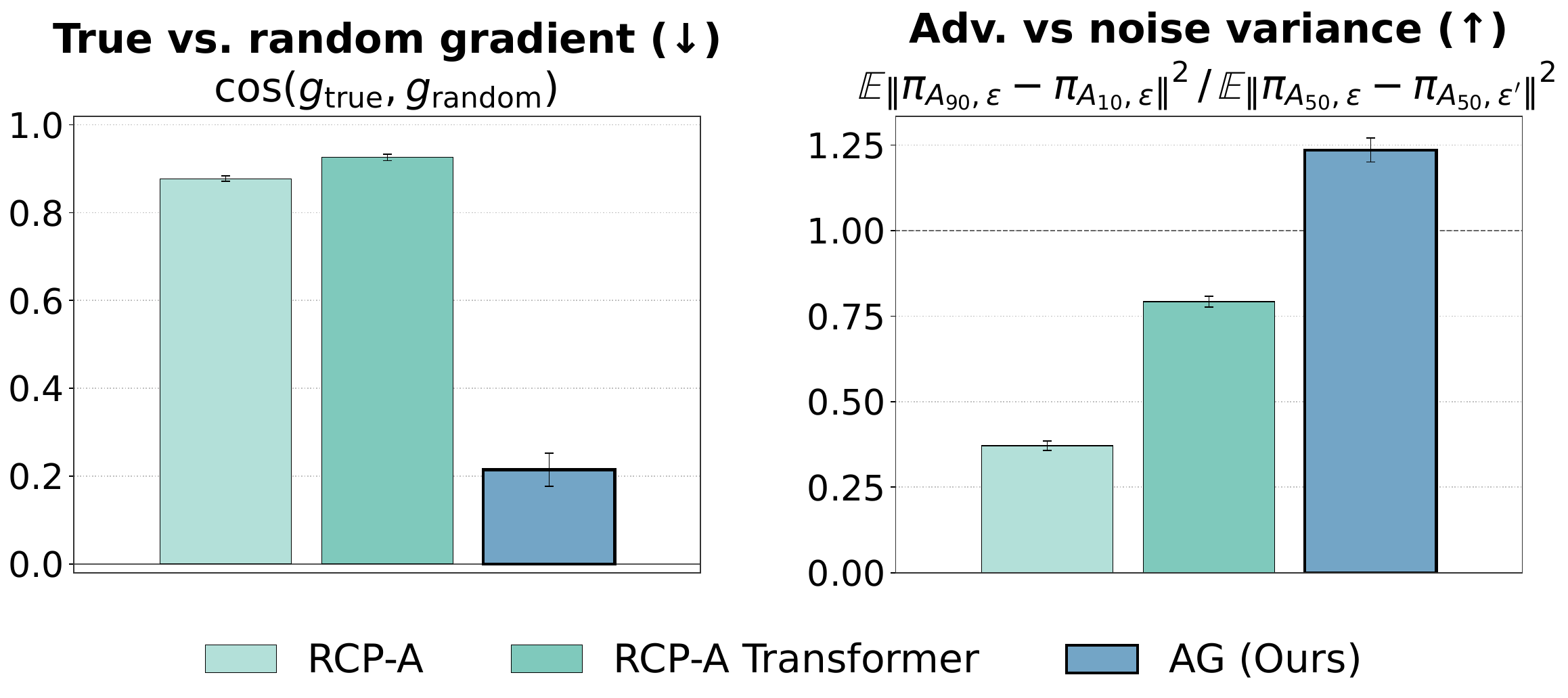}
    \includegraphics[width=0.50\linewidth]{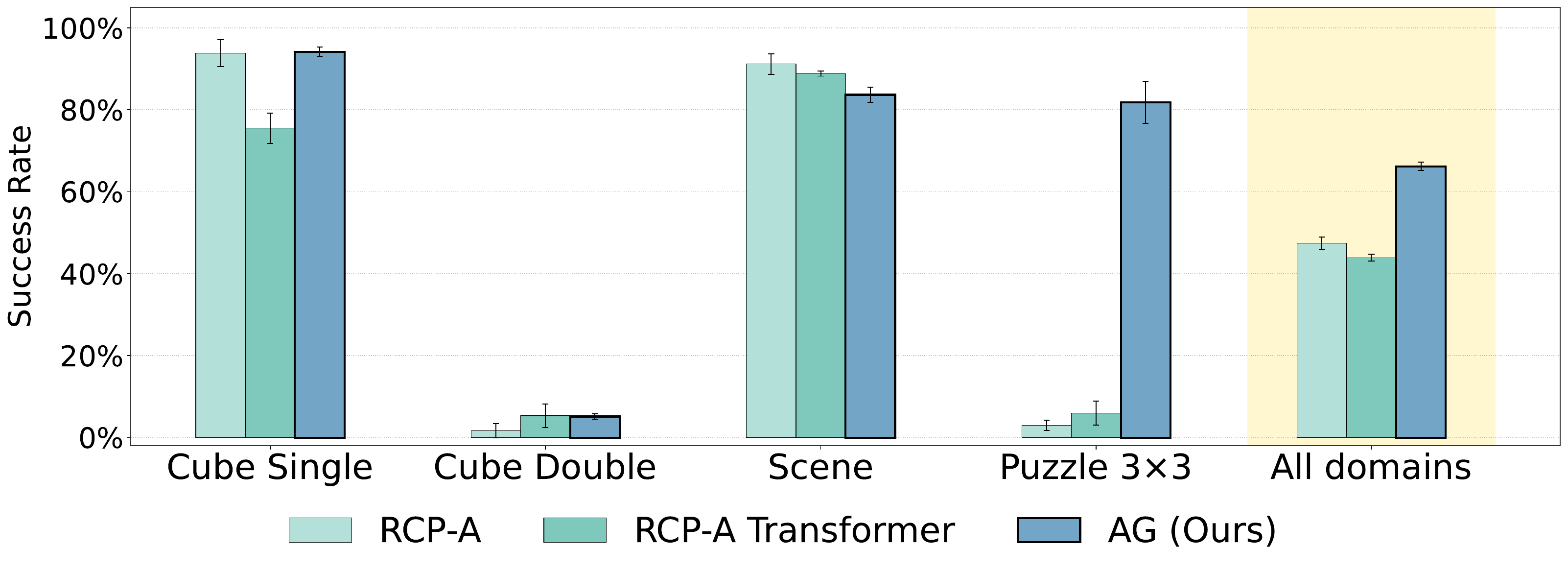}
    \caption{Using a transformer for the policy architecture does not significantly improve advantage alignment (left), and does not improve policy performance overall on OGBench (right).}
    \label{fig:rcp-transformer}
\end{figure}

\subsection{Analysis of Pure Conditional Sampling with \methodname{}}
\label{apdx:conditional-sampling}

In \Cref{sec:results-contrastive-guidance} we compared \methodname{} using different sampling schemes. In \Cref{fig:ag-conditional-eval}, we plot the performance of \methodname{} with pure conditional sampling on different environments using different query advantages.
On OGBench (top row), high query advantages leads to some performance decline, and experiences a sharp decline in \texttt{puzzle-3x3}, the longest horizon task.
On OCBench, all tasks experience a sharp drop in performance at high advantage values. This shows that high advantage value is likely out of the training distribution on OCBench where the offline dataset is demonstration-based and therefore more narrow.
\texttt{puzzle-3x3} could also suffer from the same issue because its state-action coverage is relatively smaller given its longer horizon for the same amount of offline transitions as other OGBench tasks.

\begin{figure}[h]
    \centering
    \includegraphics[width=0.8\linewidth]{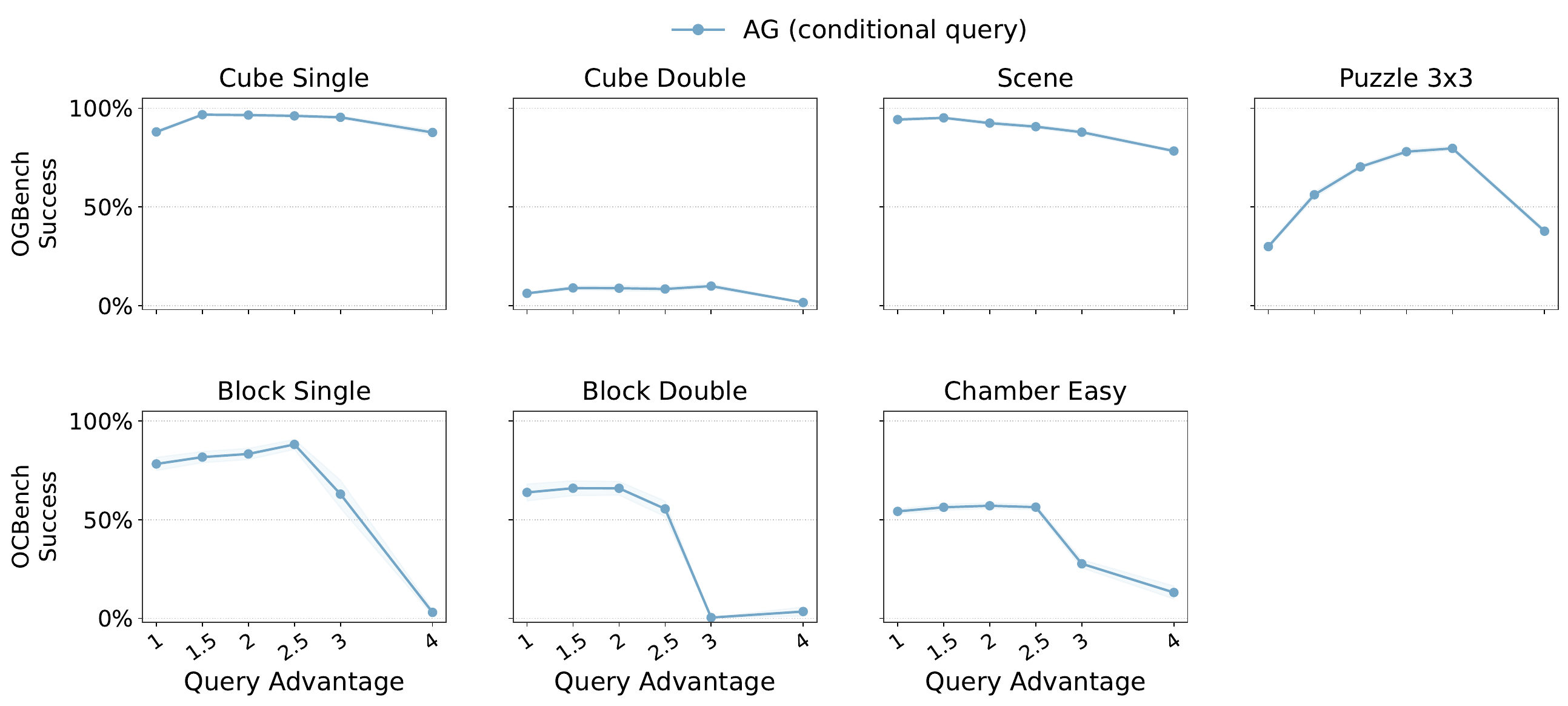}
    \caption{Pure conditional sampling of \methodname{}: querying with large advantage values leads to a sharp performance drop on OCBench and hard long horizon task \texttt{puzzle-3x3}.
    Averaged over different tasks in the same domain with $10$ seeds.
    }
    \label{fig:ag-conditional-eval}
\end{figure}

\subsection{Description of Baseline Methods}
\label{apdx:baselines}

We describe in detail the different categories of baseline methods. While some methods are implemented as Gaussian policies or diffusion policies in the original paper, we implement the flow-matching variant for all methods for fair comparison.

For all the agents below, we will assume that noise $\epsilon\sim\mathcal{N}(0,I)$, time $t\sim\mathcal{U}(0,1)$, and the expectations of training losses are taken over $(s,a)\sim\mathcal{D}$ unless otherwise specified.

\paragraph{(1) Actor-critic methods}

\paragraph{FQL}~\citep{park2025flow} learns a distilled one-step policy $\pi_\omega$ to maximize the critic, avoiding the need to take action gradient through the multi-step denoising process with backpropagation through time. Specifically, policy extraction with FQL requires learning two actors, a multi-step BC actor $\pi_\theta$, and a one-step policy $\pi_\omega$. The BC actor is trained to minimize the standard flow matching loss $\mathbb{E}\left[||v_\theta(x_t, s, t) - (a - \epsilon)||^2\right]$, while the one-step actor is trained to minimize:
\begin{equation}
    \mathbb{E}\left[
-Q_\phi(s,a^\pi)
+\alpha\left\|a^\pi-\mathrm{ODE}(v_\theta(s,\epsilon))\right\|_2^2
\right] \quad \text{where} \;\ a^\pi = v_\omega(s,\epsilon).
\end{equation}
$\mathrm{ODE}(\cdot)$ is the Euler integration process, and $\alpha$ is a coefficient that controls the behavior regularization strength.

\paragraph{EDP}~\citep{kang2023efficient} optimizes a multi-step flow policy $\pi_\theta$ to maximize the critic using a one-Euler-step approximation of the full denoising process, avoiding the full backpropagation through time. Specifically, given $x_t=(1-t)\epsilon+ta$, we approximate the full denoising process with the one-Euler step integration $\hat{a}_1(x_t)=x_t+(1-t)v_\theta(x_t,s,t)$, and train the policy to maximize $Q$-values subject to a behavior constraint:
\begin{equation}
    \mathcal{L}_{\mathrm{EDP}}(\theta)
    =
    \mathbb{E}\left[
        -Q_\phi(s,\hat{a}_1(x_t))
        + \alpha\left\|v_\theta(x_t,s,t)-(a-\epsilon)\right\|_2^2
    \right].
\end{equation}
The coefficient $\alpha$ controls the behavior regularization strength. During training, the action gradient passes through the one-step action approximation and is taken with respect to $x_t$.

\paragraph{(2) Supervised-learning-based methods}

\paragraph{CRR}~\citep{wang2020critic} uses the advantage estimates to filter the actions used for policy training. We adopt its binary-weight variant, training the flow policy only on dataset actions with positive estimated advantages:
\begin{equation}
    \mathcal{L}_{\mathrm{CRR}}(\theta)
    =
    \mathbb{E}\left[
        \mathbf{1}\{A(s,a)>0\}
        \left\|v_\theta(x_t,s,t)-(a-\epsilon)\right\|_2^2
    \right].
\end{equation}
This corresponds to behavior cloning on the positive-advantage subset. 

\paragraph{FAWR} is a baseline we construct by extending advantage weighted regression (AWR)~\citep{peng2019advantage} to flow-matching policies.
AWR is introduced for Gaussian policies, where the log likelihood of each action is weighted by its exponentiated advantage.
Unfortunately, this is not readily extendable to flow policies because the likelihood of diffusion and flow policies cannot be evaluated exactly. However, since diffusion (and under some constraints flow matching) optimize a lower bound on the log-likelihood~\citep{KingaDiffELBO}, we can optimize the lower bound of an AWR objective by directly using the flow matching loss to substitute the action log likelihood.
Specifically, we weight each dataset action's flow-matching loss by its exponentiated advantage:
\begin{equation}
    \mathcal{L}_{\mathrm{FAWR}}(\theta)
    =
    \mathbb{E}\left[
        \exp\left(\frac{A(s,a)}{\beta}\right)
        \left\|v_\theta(x_t,s,t)-(a-\epsilon)\right\|_2^2
    \right].
\end{equation}
The temperature $\beta>0$ controls how strongly training favors high-advantage actions. In practice, we clip $\exp(\frac{A(s,a)}{\beta})$ to have max value $100$.

\paragraph{RCP-R and RCP-A}~\citep{kumar2019reward} train policies conditioned on an estimated return-to-go (RCP-R) and advantage (RCP-A). While the original implementation of both RCP-A and RCP-R use Gaussian policies, we instantiate both with flow policies. Specifically, for a conditional signal $y$ (advantage or return-to-go), we optimize policy $\pi_\theta(a\mid s, y)$ with:
\begin{equation}
    \mathcal{L}_{\mathrm{RCP}}(\theta)
    =
    \mathbb{E}\left[
        \left\|v_\theta(x_t,s,t,y)-(a-\epsilon)\right\|_2^2
    \right].
\end{equation}
While the original RCP-R implementation estimates the return-to-go from $(s, a)$ using Monte Carlo estimation, we directly use the lower variance estimate $Q(s, a)$ to maintain comparability with other methods. For RCP-A, we estimate the advantage as $A(s, a) = Q(s, a) - V(s)$.
For both methods, the conditioning variable modulates the policy network's hidden features through FiLM~\citep{perez2018film}, which applies multiplicative interactions of the hidden features and the conditional variable at every layer.
Following ~\citet{kumar2019reward}, we estimate the mean $\mu_y$ and variance $\sigma_y^2$ of the conditioning variable using a running average, and, at inference time, query the policy with $y_{\mathrm{query}}=\mu_y+k\sigma_y$. The original paper always use $k=1$, whereas we sweep over different $k$ values.

\paragraph{(3) Diffusion guidance methods}

\paragraph{QGF}~\citep{zhou2026test} is a test-time RL method that trains a critic and a BC flow policy with standard flow matching loss and optimizes actions solely at test time with critic gradients guidance during sampling.
The main idea is to guide the action towards high-value directions at each denoising step using the critic's action gradient. To estimate the critic gradient at partially noised action $x_t$, QGF uses a single large Euler integration step to approximate the clean action and evaluate the critic gradient at this approximation $\hat{a}$. Specifically, QGF uses the following velocity:
\begin{equation}
    \begin{aligned}
        \hat{a}
        &= x_t+(1-t)v_{\hat{\pi}}(x_t,s,t),\\
        v_{\mathrm{QGF}}(x_t,s,t)
        &= v_{\hat{\pi}}(x_t,s,t)
        + w\left.\nabla_a Q(s,a)\right|_{a=\hat{a}}.
    \end{aligned}
\end{equation}
Here $w$ controls the guidance strength. QGF omits the Jacobian of the one-step approximation, using the critic gradient directly at $\hat{a}$. This is akin to classifier guidance of the policy using the $Q$-function as a classifier.

\paragraph{CFGRL}~\citep{frans2025diffusion} applies classifier-free guidance between an unconditional policy trained on the full dataset and a policy conditioned on positive advantage. CFGRL trains two policies: a conditional policy $\pi(a\mid s, o)$, where $o \in \{0, 1\}$ is a binary random variable that indicates optimality (\ie $o=1$ corresponds to optimal), and an unconditional policy $\pi(a\mid s)$.
Both policies are trained with flow matching. Following the original implementation, we equate the unconditional policy $\pi(a\mid s)$ with the conditional $\pi(a\mid s, o=0)$; optimality is defined as $A(s, a) = Q(s, a) - V(s) > 0$ and $\pi(a\mid s, o=1)$ is trained only on data with positive advantage estimates.
Conceptually, this is the same as training one unconditional policy on the whole dataset and another ``filtered'' policy on the positive-advantage subset.
During inference, CFGRL integrates:
\begin{equation}
    v_{\mathrm{CFGRL}}(x_t,s,t)
    =
    v_{\hat{\pi}}(x_t,s,t)
    + w\left[
        v_+(x_t,s,t)-v_{\hat{\pi}}(x_t,s,t)
    \right].
\end{equation}
The CFG weight $w$ controls the contribution of the positive-advantage policy.

\paragraph{With additional test-time compute:}

\paragraph{Best-of-N (BFN)} performs best-of-$N$ sampling of a base behavior cloning policy $\hat\pi$ trained with the flow matching objective. Specifically, given a critic $Q(s, a)$, BFN selects action
\begin{equation}
    a = \argmax_{a_1, \cdots,a_N \sim \hat{\pi}}Q(s,a_i).
\end{equation} 

\subsection{Critic Training Details}
\label{apdx:critic-training}

As described in \Cref{sec:experiments}, we follow the action chunking setting and learn both actors that produce action chunks and critics that are conditioned on action chunks. 

Our method is agnostic to how the advantage estimate is derived. For simplicity, in this paper we train value functions with IQL~\citep{kostrikov2021offline} and derive the advantage function $A(s, a) = Q(s, a) - V(s)$. The IQL objective is as follows:
\begin{align}
\mathcal{L}_{Q}(\theta)
&=
\mathbb{E}_{(s_t,a_{t:t+h},r_{t:t+h},s_{t+h})\sim\mathcal{D}}
\Big[
\big(
Q_\theta(s_t,a_{t:t+h})-(\sum_{t'=t}^{t+h-1} [\gamma^{t'-t} r_t']+\gamma^h\,V_{\bar{\psi}}(s_{t+h})
\big)^2
\Big], \\
\mathcal{L}_{V}(\psi)
&=
\mathbb{E}_{(s_t,a_{t:t+h})\sim\mathcal{D}}
\Big[
L_2^\tau\!\big(
Q_{\bar{\theta}}(s_t,a_{t:t+h})-V_\psi(s_t)
\big)
\Big],
\end{align}
where $L_2^\tau(u)=|\tau-\mathbf{1}(u<0)|u^2$ is the expectile  loss with expectile $\tau$,
and $\bar{\theta},\bar{\psi}$ are the target networks computed with EMA. We train all IQL state-action value functions with an ensemble of $2$ $Q(s, a)$ and aggregate them by taking the minimum.

\subsection{Hyperparameter and Tuning Details}
\label{apdx:hyperparameter}

The main results in \Cref{fig:main-results} and results in \Cref{fig:results-vs-bfn} and \Cref{fig:no-cfg} were run with $10$ random seeds, and all other experiments were run with $4$ random seeds unless otherwise specified. All error bars report the 95\% confidence interval.

\paragraph{Hyperparameter tunning.}
For each method, including \methodname{} and all baseline methods, we tune the task-specific hyperparameter per domain in the ranges specified in \Cref{tab:hyperparameter_ranges}. 
We tune hyperparameters with $4$ seeds, and re-run the experiments with $10$ new seeds using the chosen hyperparameter to report performance.
Specifically, on OGBench, each domain contains $5$ similar tasks of increasing hardness. We therefore tune hyperparameters per domain on tasks 2 and 4 only, and use the picked hyperparameter for all $5$ tasks on that domain. 
Tunning on task $2$ and $4$ represents an average performance on that domain since task $1$ to $5$ gradually increases in difficulty.
On OCBench, there are different numbers of tasks available per domain, and different tasks in the domain can be very different. Therefore, we tune the hyperparameter per-task per domain in OCBench.
The resulting hyperparameters we use are documented in \Cref{tab:domain_hyperparameters}.

\paragraph{Common hyperparameters.}
For all methods, we use shared hyperparemters in \Cref{tab:common_hyperparameters}. We use a small critic network size for all OGBench tasks because we find it uniformly improves performance for our and baseline methods (see \Cref{fig:critic-size-ablation}). We did not tune the critic size on OCBench.

\begin{figure}[h]
    \centering
    \includegraphics[width=\linewidth]{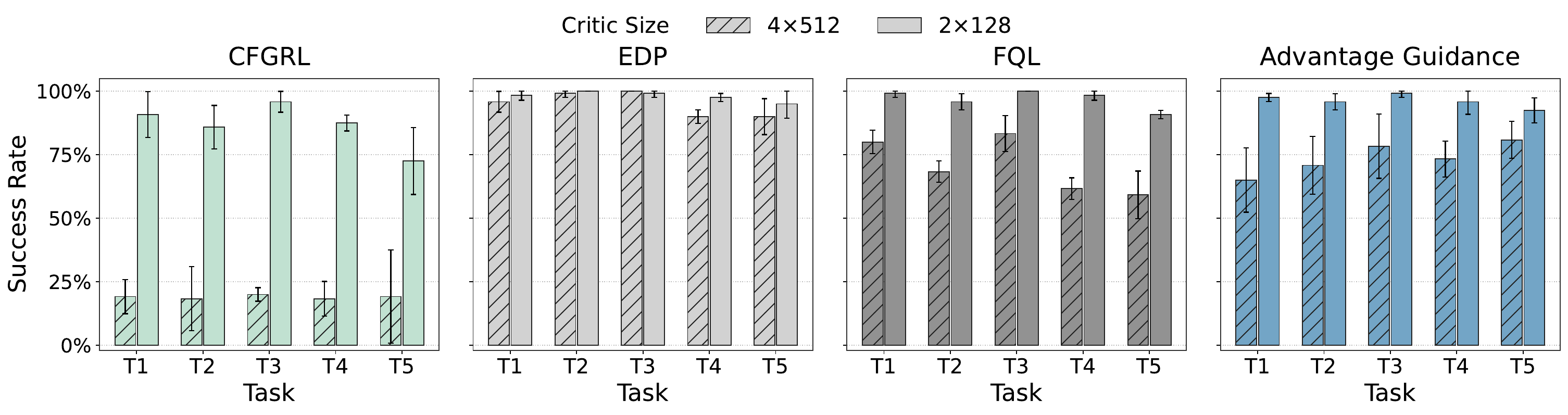}
    \caption{Comparison of two critic sizes on \texttt{cube-single} in OGBench on four different methods. $2\times 128$ significantly outperforms $4\times 512$ nearly for all methods, suggesting that the commonly used $4\times 512$ architecture suffers from critic overfitting.}
    \label{fig:critic-size-ablation}
\end{figure}

\paragraph{\methodshort{} hyperparameters.}
For \methodname{}, we use the same value per benchmark for  $\sigma_y$, and the same $y_+, y_-, K$ universally (see \Cref{tab:method_common_hyperparameters}).
This is because solely tunning the guidance weight is sufficient: 
we can observe this from the theoretical derivation in \Cref{sec:guidance-theory}, and we also empirically observe this in our ablation studies in \Cref{fig:ablation-different-label-noise} and \Cref{fig:dense-condition-pair}. We use these hyperprameters for \methodname{} in all experiments unless otherwise specified.
We use $\sigma_y=0.1$ on OGBench instead of $0.01$ because \Cref{fig:ablation-different-label-noise} shows it to be slightly more performant, though the advantage is tiny. We did not tune $\sigma_y$ on OCBench.

\begin{table}[h]
    \centering
    \small
    \renewcommand{\arraystretch}{1.12}
    \begin{tabular}{lcc}
        \toprule
        \textbf{Hyperparameter}
            & \textbf{OGBench}
            & \textbf{OCBench} \\
        \midrule
        Batch size
            & $256$
            & $1024$ \\
        Discount factor $\gamma$
            & $0.99$
            & $0.999$ \\
        IQL expectile $\tau$
            & $0.9$
            & $0.5$ \\
        Critic ensemble size 
            & 2
            & 2  \\
        Critic aggregation method
            & min 
            & min \\
        Target-network update rate
            & $0.005$
            & $0.005$ \\
        Action chunk horizon $H$
            & $5$
            & $3$ \\
        Flow integration / denoising steps
            & $10$
            & $10$ \\
        Actor hidden dimensions
            & $512 \times 4$
            & $1024 \times 8$ \\
        Critic/value hidden dimensions
            & $128 \times 2$
            & $1024 \times 8$ \\
        Actor / Critic Learning rate
            & $3\times 10^{-4}$
            & $1\times 10^{-4}$ \\
        Optimizer
            & Adam
            & Adam \\
        Offline training steps
            & $5\times 10^{5}$
            & $1\times 10^{6}$ ($1 \times 10^{5}$ for AG on Block Single) \\
        Evaluation episodes
            & $30$
            & $500$ \\
        \bottomrule
    \end{tabular}
    \caption{
        \textbf{Common hyperparameters} used for the OGBench and OCBench experiments in \Cref{fig:main-results}.
    }
    \label{tab:common_hyperparameters}
\end{table}

\begin{table}[h]
    \centering
    \renewcommand{\arraystretch}{1.15}
    \setlength{\tabcolsep}{12pt}
    \begin{tabular}{lcc}
        \toprule
        \textbf{Hyperparameter}
            & \textbf{OGBench}
            & \textbf{OCBench} \\
        \midrule
        Positive condition $y_{+}$
            & $1$
            & $1$ \\
        Negative condition $y_{-}$
            & $-1$
            & $-1$ \\
        Advantage label standard deviation $\sigma_y$
            & $0.1$
            & $0.01$ \\
        Number of Imagined Actions $K$ 
            & $7$ 
            & $7$  \\
        \bottomrule
    \end{tabular}

    \caption{
        \textbf{Common hyperparameters for \methodname{}}. As shown in \Cref{fig:ablation-different-label-noise} and \Cref{fig:dense-condition-pair}, different values of $y_+, y_-, \sigma_y$ hyperparameters lead to very similar results. We therefore fix these parameters across tasks on the same benchmark and only tune the guidance weight for \methodname{}.
    }
    \label{tab:method_common_hyperparameters}
\end{table}

\begin{table*}[h]
    \centering
    \renewcommand{\arraystretch}{1.15}
    \setlength{\tabcolsep}{4pt}
    \resizebox{\textwidth}{!}{%
    \begin{tabular}{@{}ll*{9}{c}@{}}
        \toprule
        \textbf{Suite}
        & \textbf{Domain}
        & \textbf{EDP ($\alpha_\mathrm{EDP}$})
        & \textbf{FQL ($\alpha_\mathrm{FQL}$)}
        & \textbf{QGF ($w_\mathrm{QGF}$)}
        & \textbf{CFGRL ($w_\mathrm{CFGRL}$)}
        & \textbf{FAWR ($\beta$)}
        & \textbf{RCP-A ($k$)}
        & \textbf{RCP-R ($k$)}
        & \textbf{CRR}
        & \textbf{AG ($w$)} \\
        \midrule

        OGBench & Cube Single
            & $100$ & $1000$ & $0.08$ & $2$ & $0.3$
            & 1 & 1 & -- & $1$ \\

        OGBench & Cube Double
            & $10$ & $100$ & $0.16$ & $2$ & $0.3$
            & 1 & 1 & -- & $1.5$ \\

        OGBench & Scene
            & $30$ & $300$ & $0.16$ & $2$ & $10$
            & 1 & 1 & -- & $1$ \\

        OGBench & Puzzle $3{\times}3$
            & $30$ & $300$ & $0.16$ & $3$ & $0.3$
            & 1 & 1 & -- & $2.5$ \\

        \midrule

        OCBench & Block Single
            & $1$ & $300$ & $0.02$ & $10$ & $3$
            & 1 & 1 & -- & $1.5$ \\

        OCBench & Block Double
            & $(1,1)$ & $(100,100)$ & $(0.32,0.32)$ & $(15,15)$
            & $(3,0.3)$ & 1 & 1 & -- & $(3,4)$ \\

        OCBench & Chamber Easy
            & $(1,1,10)$ & $(100,100,100)$ & $(0.32,0.32,0.32)$
            & $(5,15,10)$ & $(1,1,1)$
            & 1 & 1 & -- & $(4,3,2)$ \\

        \bottomrule
    \end{tabular}%
    }

    \caption{
        \textbf{Domain-specific hyperparameters} used for OGBench
        and OCBench results in \Cref{fig:main-results}. For OCBench,
        parenthesized values are ordered by task number. For example, $(3,4)$
        denotes the values used for Tasks 1 and 2 of Block Double. A dash indicates that the method has no tuned domain-specific hyperparameter.
    }
    \label{tab:domain_hyperparameters}
\end{table*}

\begin{table*}[h]
    \centering
    \renewcommand{\arraystretch}{1.18}
    \resizebox{\textwidth}{!}{%
    \begin{tabular}{lll}
        \toprule
        \textbf{Method}
            & \textbf{Hyperparameter}
            & \textbf{Tunning Range} \\
        \midrule

        EDP
            & BC coefficient $\alpha_\mathrm{EDP}$
            & \makecell[l]{
                OGBench: $\{10,30,100,300\}$; \quad
                OCBench: $\{1,10\}$ 
              } \\

        FQL
            & BC coefficient $\alpha_\mathrm{FQL}$
            & $\{100,300,1000\}$ \\

        QGF
            & Guidance weight $w$
            & $\{0,0.002,0.004,0.01,0.02,0.04,0.08,0.16,0.32\}$ \\

        CFGRL
            & Guidance weight $w$
            & \makecell[l]{
                OGBench: $\{0,0.5,1,2,3,5,10,15\}$;\quad
                OCBench: $\{0,0.5,1,1.5,2,3,5,10,15\}$
              } \\

        FAWR
            & Inverse temperature $\beta$
            & \makecell[l]{
                OGBench: $\{0.3,1,3,10\}$; \quad
                OCBench: $\{0.3,1,3\}$
              } \\

        RCP-A
            & inference std multiplier $k$
            & \makecell[l]{
                OGBench: $\{ 0.5, 0.7, 1, 1.5, 2, 3\}$;
                OCBench: $\{1\}$
            } \\

        AG (Ours)
            & Guidance weight $w$
            & \makecell[l]{
                OGBench: $\{0,0.75,1,1.25,1.5,2,2.5,4\}$;\quad
                OCBench: $\{0,0.5,1,1.5,2,3,4\}$
              } \\

        \bottomrule
    \end{tabular}%
    }
    \caption{
        Hyperparameter ranges for OGBench and OCBench experiments.
    }
    \label{tab:hyperparameter_ranges}
\end{table*}

\subsection{Tasks and Environment Details}
\label{apdx:tasks-envs}

We use all the available task for each domain in two of the benchmarks. OCBench offers different numbers of tasks per domain, ranging from $1$ to $3$.
OGBench tasks all have offline datasets with $1$M transitions. OCBench block single and chamber easy has $10$k offline transitions, while block double has $30$k transitions. Note that while OCBench offers demonstration-based datasets, a small fraction of the data is unsuccessful demonstrations.

\clearpage
\subsection{Detailed Results}
\label{apdx:detailed-results}

\begin{figure}[h]
    \centering
    \includegraphics[width=\linewidth]{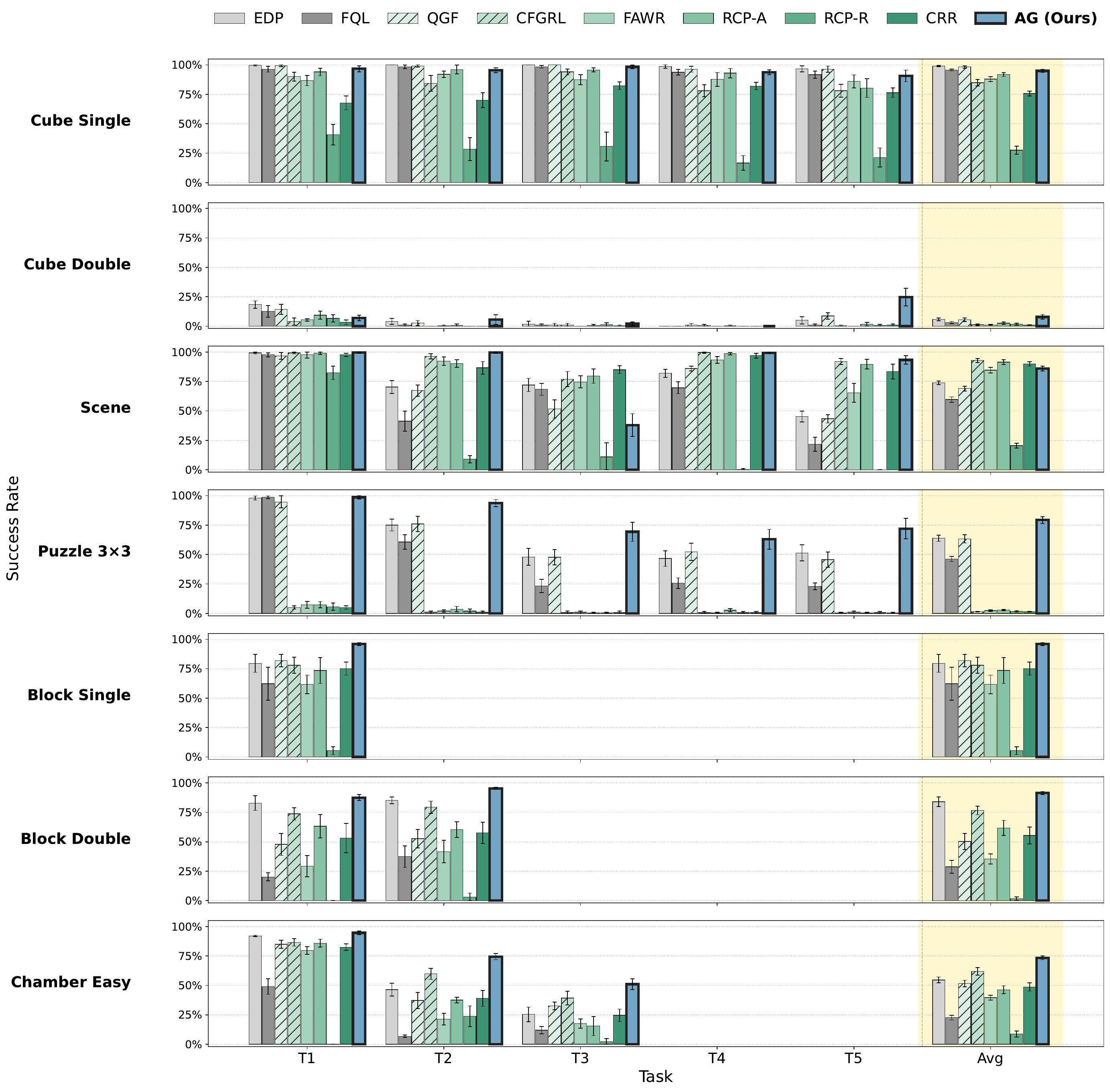}
    \caption{Task breakdown of \Cref{fig:main-results}: Comparing \methodname{} against prior methods in offline RL setting on 20 OGBench tasks and 6 OCBench tasks. Averaged over 10 random seeds. See exact plotted numbers in \Cref{tab:main-result-details}.}
    \label{fig:main-results-details}
\end{figure}

\definecolor{agblue}{HTML}{73A5C6}
\definecolor{aglightblue}{HTML}{EDF5FA}
\definecolor{agdarkblue}{HTML}{0076BA}

\newcommand{\best}[1]{$\mathbf{#1}$}
\newcommand{\agval}[1]{%
  \cellcolor{aglightblue}\textcolor{agblue}{$#1$}}
\newcommand{\agbest}[1]{%
  \cellcolor{aglightblue}\textcolor{agdarkblue}{$\mathbf{#1}$}}

\begin{table}[t]
    \centering
    \small
    \setlength{\tabcolsep}{2.5pt}
    \renewcommand{\arraystretch}{1.12}

    \begin{adjustbox}{max width=\linewidth}
    \begin{tabular}{@{}llrrrrrrrrr@{}}
        \toprule
        \textbf{Environment} & \textbf{Task} & \textbf{EDP} & \textbf{FQL} & \textbf{QGF} &
        \textbf{CFGRL} & \textbf{FAWR} & \textbf{RCP-A} & \textbf{RCP-R} & \textbf{CRR} &
        \textcolor{agdarkblue}{\textbf{AG}} \\
        \midrule

        \multirow{6}{*}{Cube Single}
        & T1  & \best{99.7 \pm 0.7} & \best{96.3 \pm 2.5} & \best{99.3 \pm 0.9} & $90.0 \pm 3.8$ & $86.7 \pm 4.4$ & $94.0 \pm 3.2$ & $40.7 \pm 8.6$ & $67.7 \pm 5.8$ & \agbest{96.7 \pm 2.6} \\
        & T2  & \best{100.0 \pm 0.0} & \best{98.3 \pm 1.5} & \best{99.0 \pm 1.0} & $84.3 \pm 6.5$ & $92.0 \pm 2.8$ & \best{96.0 \pm 3.7} & $28.3 \pm 9.8$ & $70.0 \pm 6.4$ & \agbest{95.3 \pm 2.0} \\
        & T3  & \best{100.0 \pm 0.0} & \best{98.3 \pm 1.1} & \best{100.0 \pm 0.0} & $94.0 \pm 2.5$ & $87.3 \pm 4.3$ & \best{95.7 \pm 1.7} & $30.7 \pm 12.4$ & $82.3 \pm 3.1$ & \agbest{98.3 \pm 1.5} \\
        & T4  & \best{98.7 \pm 1.4} & $93.7 \pm 2.5$ & \best{96.3 \pm 2.5} & $78.0 \pm 5.1$ & $87.7 \pm 5.8$ & $93.0 \pm 4.0$ & $16.7 \pm 6.2$ & $82.0 \pm 3.1$ & \agval{93.7 \pm 2.1} \\
        & T5  & \best{96.7 \pm 2.6} & $91.7 \pm 3.1$ & \best{96.3 \pm 2.7} & $78.0 \pm 5.4$ & $86.0 \pm 5.5$ & $80.3 \pm 7.9$ & $21.3 \pm 8.1$ & $76.3 \pm 4.0$ & \agval{90.7 \pm 4.9} \\
        & Avg & \best{99.0 \pm 0.6} & \best{95.7 \pm 0.8} & \best{98.2 \pm 1.0} & $84.9 \pm 2.7$ & $87.9 \pm 2.0$ & $91.8 \pm 1.7$ & $27.5 \pm 3.5$ & $75.7 \pm 2.0$ & \agbest{94.9 \pm 1.2} \\
        \midrule

        \multirow{6}{*}{Cube Double}
        & T1  & \best{18.3 \pm 3.1} & $12.7 \pm 4.9$ & $14.3 \pm 4.4$ & $4.0 \pm 3.0$ & $5.3 \pm 1.1$ & $9.3 \pm 3.5$ & $6.7 \pm 3.1$ & $3.3 \pm 1.9$ & \agval{7.0 \pm 2.5} \\
        & T2  & $4.0 \pm 2.5$ & $1.0 \pm 1.0$ & $2.7 \pm 2.1$ & $0.0 \pm 0.0$ & $0.3 \pm 0.7$ & $0.7 \pm 1.3$ & $0.0 \pm 0.0$ & $0.0 \pm 0.0$ & \agbest{5.7 \pm 4.0} \\
        & T3  & $1.7 \pm 2.6$ & $1.0 \pm 1.0$ & $1.0 \pm 1.4$ & $1.0 \pm 1.4$ & $0.0 \pm 0.0$ & $0.7 \pm 0.9$ & $1.3 \pm 1.4$ & $0.3 \pm 0.7$ & \agbest{2.0 \pm 1.4} \\
        & T4  & $0.0 \pm 0.0$ & $0.0 \pm 0.0$ & \best{1.0 \pm 1.4} & $0.7 \pm 0.9$ & $0.0 \pm 0.0$ & $0.3 \pm 0.7$ & $0.0 \pm 0.0$ & $0.0 \pm 0.0$ & \agval{0.0 \pm 0.0} \\
        & T5  & $5.0 \pm 3.1$ & $1.0 \pm 1.0$ & $8.7 \pm 2.8$ & $0.3 \pm 0.7$ & $0.0 \pm 0.0$ & $1.7 \pm 1.5$ & $0.7 \pm 0.9$ & $1.0 \pm 1.0$ & \agbest{24.7 \pm 7.4} \\
        & Avg & $5.8 \pm 1.2$ & $3.1 \pm 1.0$ & $5.5 \pm 1.5$ & $1.2 \pm 0.7$ & $1.1 \pm 0.3$ & $2.5 \pm 1.0$ & $1.7 \pm 0.9$ & $0.9 \pm 0.3$ & \agbest{7.9 \pm 2.0} \\
        \midrule

        \multirow{6}{*}{Scene}
        & T1  & \best{99.3 \pm 0.9} & \best{97.7 \pm 1.7} & \best{96.7 \pm 2.9} & \best{99.3 \pm 0.9} & \best{97.7 \pm 2.8} & \best{99.0 \pm 1.0} & $82.3 \pm 5.7$ & \best{97.7 \pm 1.4} & \agbest{99.7 \pm 0.7} \\
        & T2  & $70.3 \pm 5.4$ & $41.3 \pm 8.6$ & $67.3 \pm 4.8$ & \best{96.3 \pm 2.3} & $92.3 \pm 3.7$ & $90.3 \pm 3.3$ & $9.0 \pm 3.1$ & $86.7 \pm 5.3$ & \agbest{99.7 \pm 0.7} \\
        & T3  & $72.0 \pm 5.6$ & $68.3 \pm 5.0$ & $51.7 \pm 7.8$ & $77.0 \pm 6.4$ & $74.7 \pm 5.1$ & $79.7 \pm 6.1$ & $11.0 \pm 12.0$ & \best{85.0 \pm 3.3} & \agval{38.0 \pm 9.6} \\
        & T4  & $82.0 \pm 3.4$ & $69.7 \pm 5.0$ & $86.0 \pm 2.1$ & \best{99.7 \pm 0.7} & $93.3 \pm 2.9$ & \best{98.7 \pm 1.1} & $0.3 \pm 0.7$ & \best{97.0 \pm 2.1} & \agbest{99.3 \pm 0.9} \\
        & T5  & $45.3 \pm 4.6$ & $21.7 \pm 6.0$ & $43.3 \pm 3.5$ & \best{92.0 \pm 2.6} & $65.3 \pm 8.0$ & \best{89.7 \pm 4.1} & $0.0 \pm 0.0$ & $83.3 \pm 6.3$ & \agbest{93.3 \pm 3.5} \\
        & Avg & $73.8 \pm 1.5$ & $59.7 \pm 2.3$ & $69.0 \pm 2.0$ & \best{92.9 \pm 1.7} & $84.7 \pm 2.4$ & \best{91.5 \pm 2.0} & $20.5 \pm 2.1$ & \best{89.9 \pm 1.7} & \agval{86.0 \pm 2.0} \\
        \midrule

        \multirow{6}{*}{Puzzle $3\times3$}
        & T1  & \best{98.0 \pm 1.7} & \best{98.7 \pm 1.1} & \best{94.7 \pm 5.1} & $5.0 \pm 1.5$ & $7.3 \pm 2.9$ & $7.3 \pm 2.7$ & $5.7 \pm 3.1$ & $5.0 \pm 1.5$ & \agbest{98.7 \pm 1.4} \\
        & T2  & $75.0 \pm 5.0$ & $60.7 \pm 6.1$ & $76.0 \pm 6.6$ & $1.0 \pm 1.0$ & $2.0 \pm 1.1$ & $3.7 \pm 2.3$ & $2.0 \pm 1.7$ & $1.0 \pm 1.0$ & \agbest{93.7 \pm 3.0} \\
        & T3  & $48.0 \pm 7.2$ & $23.3 \pm 5.7$ & $47.7 \pm 6.5$ & $0.7 \pm 1.3$ & $1.0 \pm 1.0$ & $0.3 \pm 0.7$ & $0.3 \pm 0.7$ & $0.7 \pm 1.3$ & \agbest{69.3 \pm 8.1} \\
        & T4  & $46.7 \pm 6.6$ & $25.7 \pm 4.5$ & $52.3 \pm 7.4$ & $0.7 \pm 0.9$ & $0.3 \pm 0.7$ & $2.7 \pm 1.3$ & $0.7 \pm 0.9$ & $0.7 \pm 0.9$ & \agbest{63.0 \pm 8.4} \\
        & T5  & $51.3 \pm 6.8$ & $23.0 \pm 3.0$ & $45.7 \pm 6.5$ & $0.3 \pm 0.7$ & $1.0 \pm 1.0$ & $0.3 \pm 0.7$ & $0.7 \pm 0.9$ & $0.3 \pm 0.7$ & \agbest{72.0 \pm 8.7} \\
        & Avg & $63.8 \pm 2.6$ & $46.3 \pm 2.2$ & $63.3 \pm 3.5$ & $1.5 \pm 0.3$ & $2.3 \pm 0.7$ & $2.9 \pm 0.4$ & $1.9 \pm 0.7$ & $1.5 \pm 0.3$ & \agbest{79.3 \pm 2.9} \\
        \midrule

        \multirow{2}{*}{Block Single}
        & T1  & $79.5 \pm 7.6$ & $62.2 \pm 14.0$ & $81.8 \pm 5.2$ & $77.9 \pm 6.9$ & $61.6 \pm 8.0$ & $73.5 \pm 11.1$ & $5.5 \pm 3.3$ & $75.1 \pm 5.5$ & \agbest{95.9 \pm 1.3} \\
        & Avg & $79.5 \pm 7.6$ & $62.2 \pm 14.0$ & $81.8 \pm 5.2$ & $77.9 \pm 6.9$ & $61.6 \pm 8.0$ & $73.5 \pm 11.1$ & $5.5 \pm 3.3$ & $75.1 \pm 5.5$ & \agbest{95.9 \pm 1.3} \\
        \midrule

        \multirow{3}{*}{Block Double}
        & T1  & $82.9 \pm 6.3$ & $20.1 \pm 3.4$ & $47.8 \pm 9.3$ & $73.7 \pm 5.2$ & $29.3 \pm 9.0$ & $63.2 \pm 10.0$ & $0.0 \pm 0.0$ & $53.2 \pm 12.5$ & \agbest{87.5 \pm 2.6} \\
        & T2  & $85.2 \pm 3.0$ & $37.3 \pm 9.1$ & $52.6 \pm 7.9$ & $79.3 \pm 5.2$ & $41.6 \pm 9.5$ & $60.3 \pm 6.6$ & $3.1 \pm 3.3$ & $57.5 \pm 9.1$ & \agbest{95.3 \pm 0.9} \\
        & Avg & $84.0 \pm 4.0$ & $28.7 \pm 5.4$ & $50.2 \pm 6.8$ & $76.5 \pm 3.6$ & $35.4 \pm 4.2$ & $61.8 \pm 6.4$ & $1.6 \pm 1.6$ & $55.4 \pm 7.2$ & \agbest{91.4 \pm 1.3} \\
        \midrule

        \multirow{4}{*}{Chamber Easy}
        & T1  & \best{92.0 \pm 0.6} & $49.0 \pm 6.5$ & $84.9 \pm 3.5$ & $86.5 \pm 3.2$ & $79.6 \pm 3.2$ & $85.8 \pm 3.4$ & $0.1 \pm 0.1$ & $82.4 \pm 2.7$ & \agbest{94.6 \pm 1.5} \\
        & T2  & $46.4 \pm 5.5$ & $6.8 \pm 1.1$ & $37.2 \pm 6.8$ & $59.8 \pm 4.7$ & $21.4 \pm 5.0$ & $37.5 \pm 2.3$ & $23.8 \pm 8.8$ & $39.0 \pm 6.8$ & \agbest{74.4 \pm 2.6} \\
        & T3  & $25.4 \pm 6.2$ & $11.9 \pm 3.1$ & $32.6 \pm 3.3$ & $39.2 \pm 5.9$ & $17.6 \pm 3.8$ & $15.5 \pm 8.2$ & $2.2 \pm 2.5$ & $24.6 \pm 5.3$ & \agbest{51.0 \pm 4.6} \\
        & Avg & $54.6 \pm 2.5$ & $22.6 \pm 2.0$ & $51.6 \pm 2.8$ & $61.8 \pm 3.2$ & $39.5 \pm 1.9$ & $46.3 \pm 3.3$ & $8.7 \pm 2.6$ & $48.7 \pm 3.5$ & \agbest{73.3 \pm 1.6} \\

        \bottomrule
    \end{tabular}
    \end{adjustbox}

    \caption{Numeric details of \Cref{fig:main-results-details}. Results report mean $\pm$ 95\% confidence intervals across 10 random seeds. Bold entries indicate means at least 95\% of the best mean in each comparison group.}
    \label{tab:main-result-details}

\end{table}

\begin{figure}[h]
    \centering
    \includegraphics[width=0.95\linewidth]{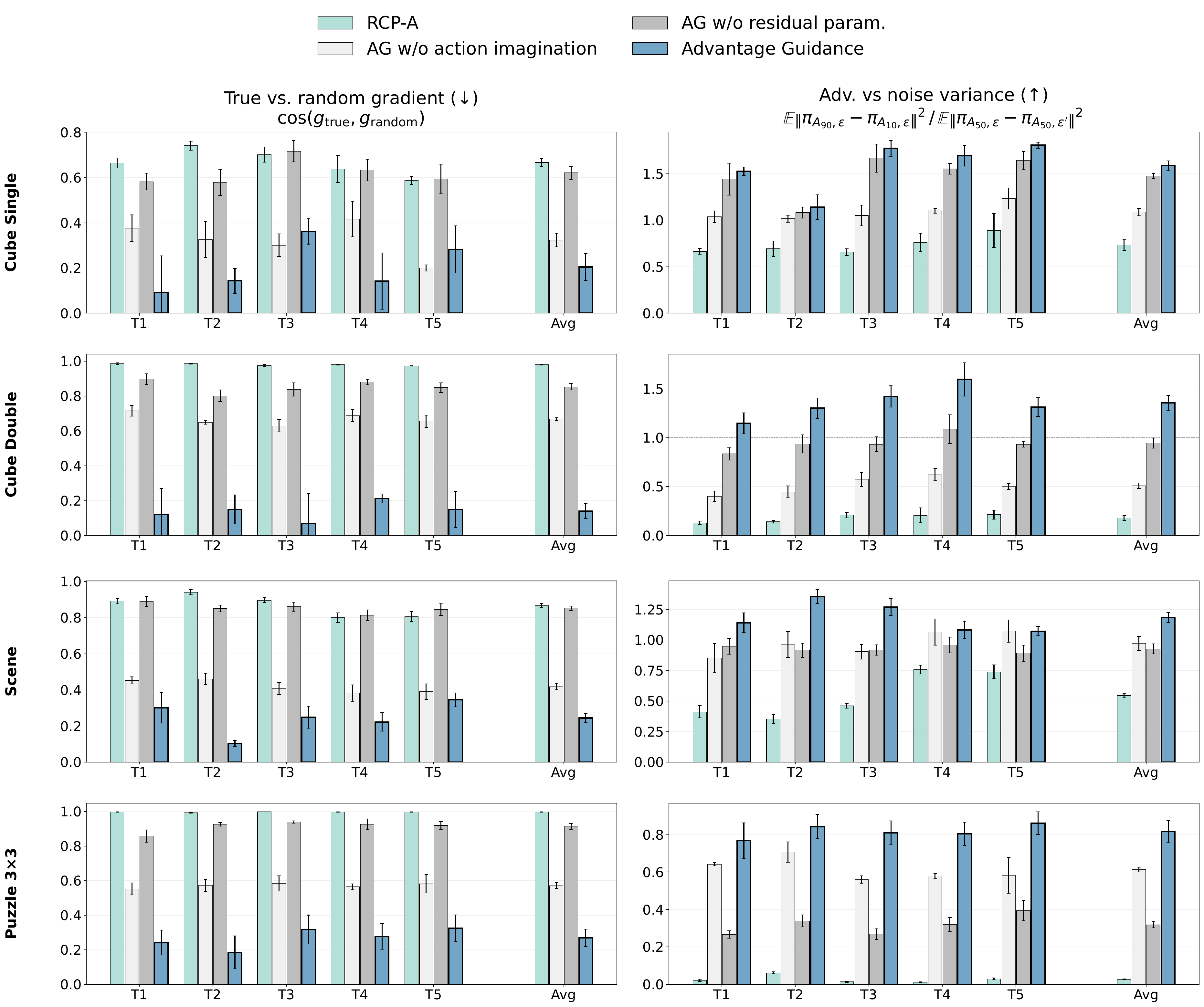}
    \caption{Task breakdown of \Cref{fig:ablate-residual-aug-metrics}: Comparing (1) gradient similarity between gradient of policy loss conditioned on true and random advantage values and (2) the variance of policy output conditioned on different advantage values or initial noise.}
    \label{fig:attention-challenges-detail}
\end{figure}

\begin{figure}[h]
    \centering
    \includegraphics[width=0.9\linewidth]{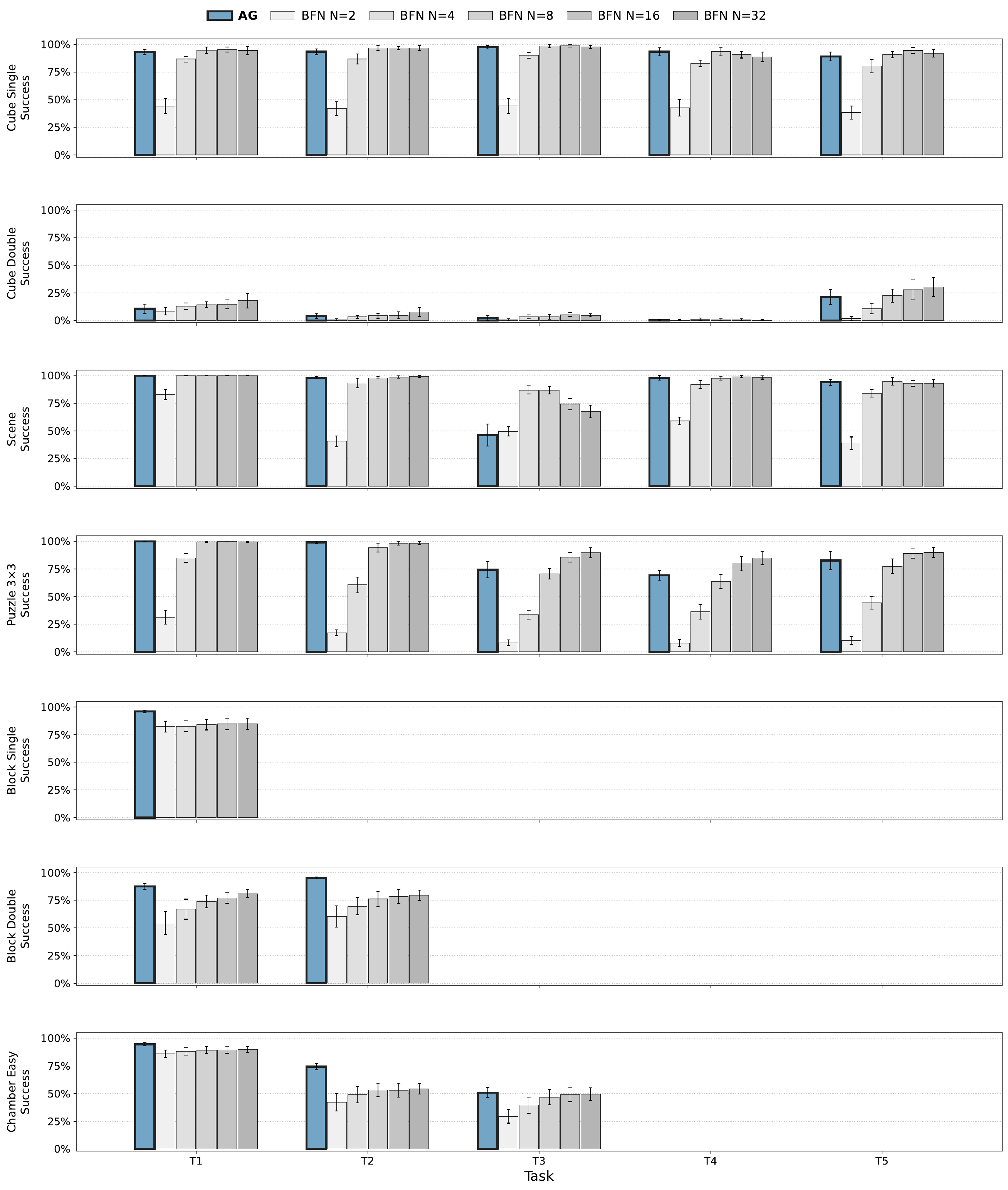}
    \caption{Task breakdown of \Cref{fig:results-vs-bfn}, comparing \methodname{} with best-of-N (BFN) sampling at different values of $N$ on OGBench and OCBench. Averaged over $10$ random seeds.}
    \label{fig:bfn-details}
\end{figure}

\begin{figure}[h]
    \centering
    \includegraphics[width=0.95\linewidth]{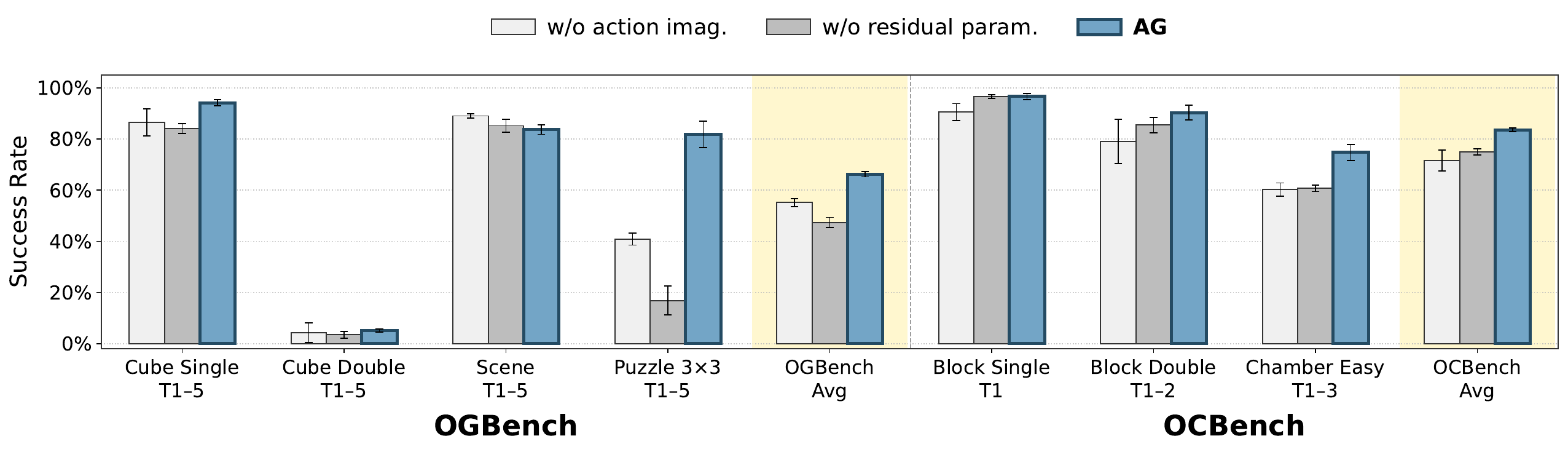}
    \caption{Domain breakdown of \Cref{fig:ablate-residual-aug-performance}: ablation of residual parameterization and action imagination.}
    \label{fig:ablate-residual-aug-details}
\end{figure}

\begin{figure}[h]
    \centering
    \includegraphics[width=0.9\linewidth]{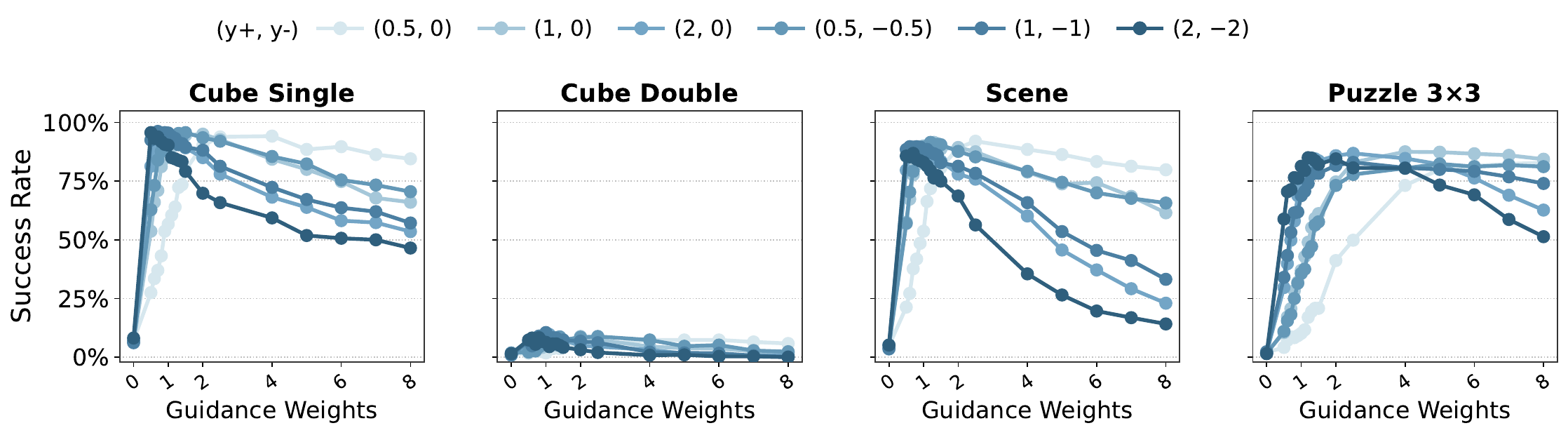}
    \includegraphics[width=0.6\linewidth]{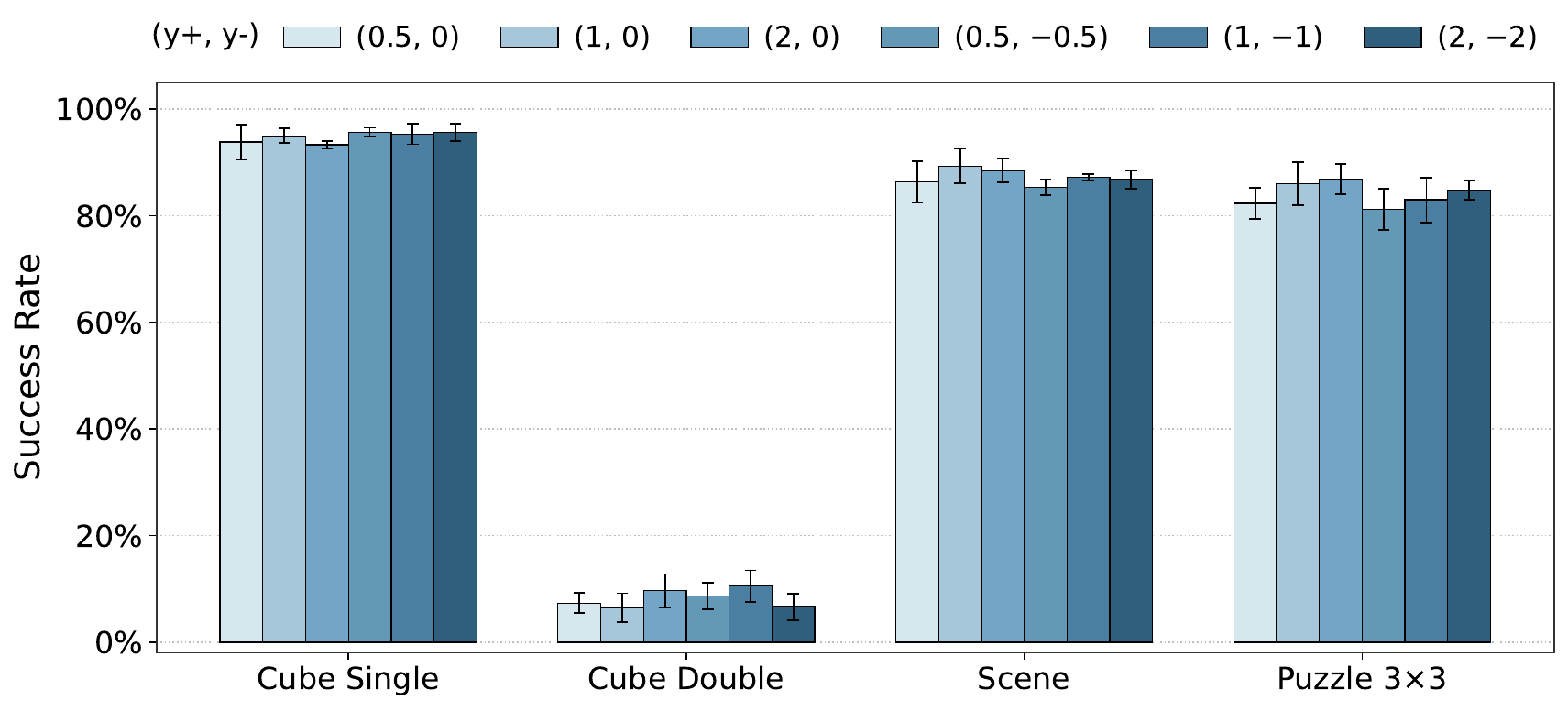}
    \caption{Task breakdown of \Cref{fig:dense-condition-pair}: sensitivity to contrastive condition pair $(y_+, y_-)$ and guidance weight $w$. \textit{Top:} Performance at different guidance weights. \textit{Bottom:} Performance at best guidance weight.}
    \label{fig:dense-condition-pair-details}
\end{figure}

\begin{figure}[h]
    \centering
    \includegraphics[width=0.95\linewidth]{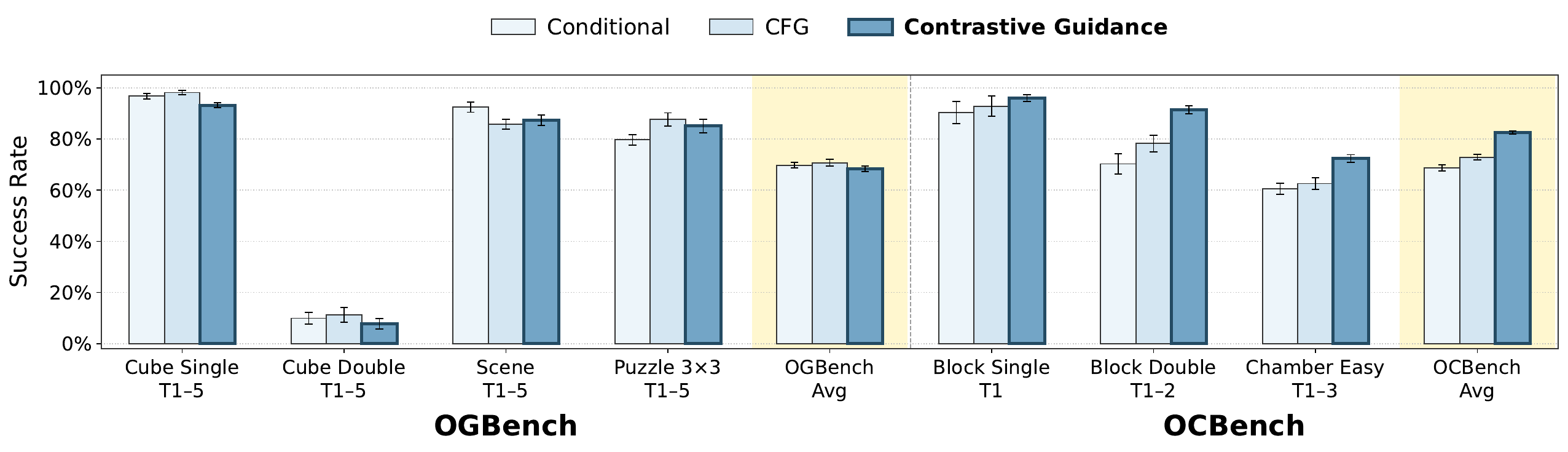}
    \caption{Domain breakdown of \Cref{fig:no-cfg}: Ablation of sampling method for \methodname{}. 10 random seeds.}
    \label{fig:no-cfg-domain-details}
\end{figure}

\begin{figure}[h]
    \centering
    \includegraphics[width=0.6\linewidth]{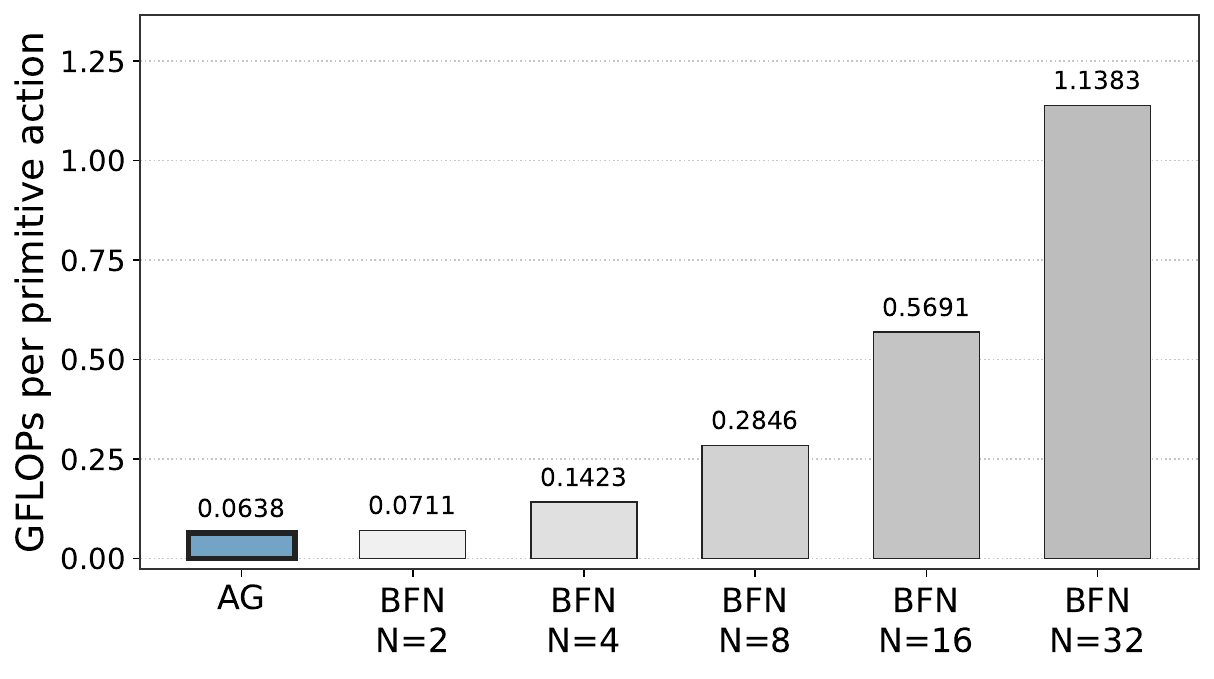}
    \caption{Compute Cost of BFN agents. The result is measured in terms of FLOPs per primitive action in the action chunk, and is averaged over OGBench and OCBench environments.}
    \label{fig:bfn-compute-cost}
\end{figure}

\clearpage
\subsection{Additional Results}
\label{apdx:additional-results}

\paragraph{Higher discount.}
\Cref{fig:main-results} shows that all methods perform poorly in long-horizon sparse reward \texttt{cube-double} tasks from OGBench. Here, we experiment whether a higher discount factor can improve performance. In \Cref{fig:higher-discount}, we use $\gamma=0.999$ instead of $0.99$, and find that all agents benefit from a larger discount. QGF and \methodname{} in particular benefit the most, and both significantly outperform the other baseline methods.

\begin{figure}[h]
    \centering
    \includegraphics[width=0.9\linewidth]{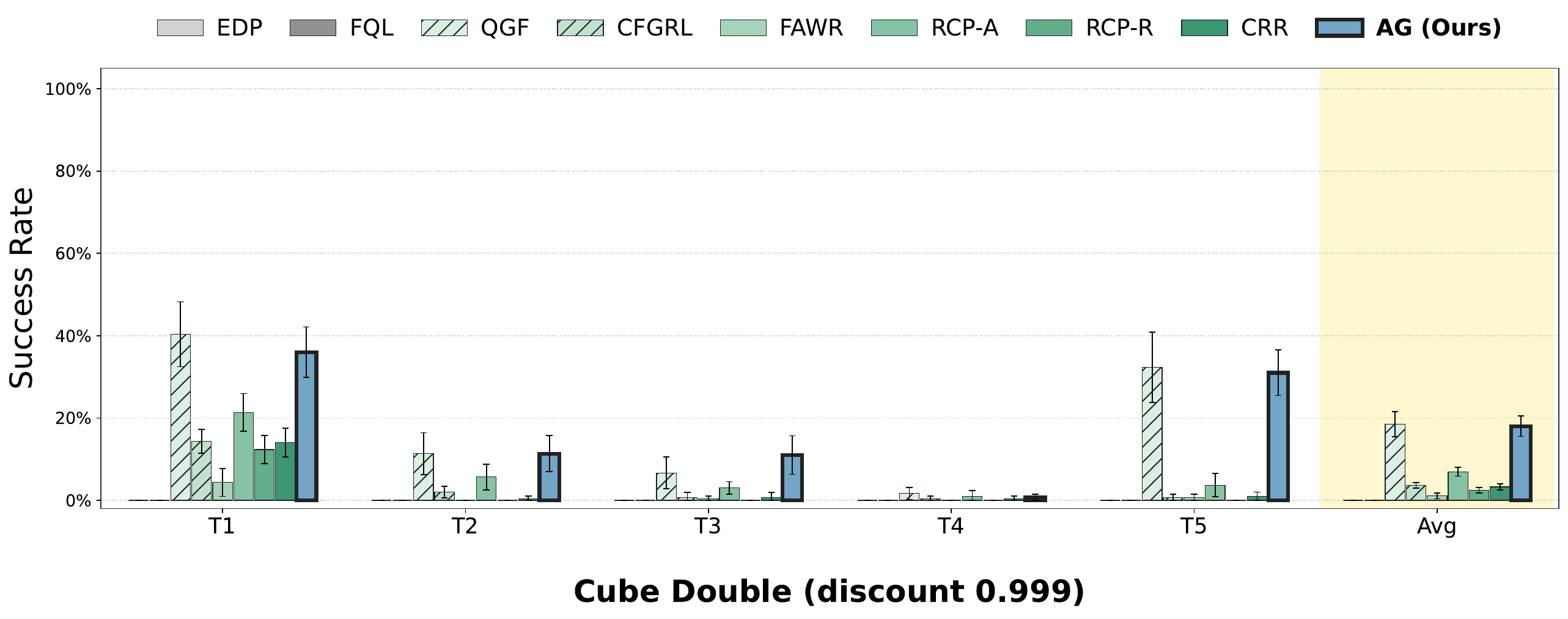}
    \caption{Offline RL performance on \texttt{cube-double} with discount factor $0.999$.}
    \label{fig:higher-discount}
\end{figure}

\end{document}